\documentclass{article}

\usepackage[final]{neurips_2022}

\usepackage[utf8]{inputenc} % allow utf-8 input
\usepackage[T1]{fontenc}    % use 8-bit T1 fonts
\usepackage{hyperref}       % hyperlinks
\usepackage{url}            % simple URL typesetting
\usepackage{booktabs}       % professional-quality tables
\usepackage{amsfonts}       % blackboard math symbols
\usepackage{nicefrac}       % compact symbols for 1/2, etc.
\usepackage{microtype}      % microtypography
\usepackage{xcolor}         % colors
\usepackage{hyperref}
\usepackage{amsfonts}
\usepackage{xcolor}
\usepackage{amsmath}
\usepackage{amsthm}
\usepackage{graphicx}
\usepackage{algorithm}
\usepackage{algpseudocode}
\usepackage{caption}
\usepackage{subcaption}
\usepackage{appendix}

\newtheorem{prop}{Proposition}[section]

\newcommand{\mainfigsize}{0.22}
\newcommand{\resultfigsize}{0.13}
\newcommand{\sacfigsize}{0.165}
\newcommand{\appendixresultfigsize}{0.25}
\newcommand{\LB}{\mathrm{LB}}
\newcommand{\UB}{\mathrm{UB}}

\newcommand{\etal}{\textit{et al.}}

\title{Off-policy Reinforcement Learning with Optimistic Exploration and Distribution Correction}

\title{Off-policy Reinforcement Learning with Optimistic Exploration and Distribution Correction}
\author{%
  Jiachen Li \\
  UC Santa Barbara\\
  \And
  Shuo Cheng \\
  Georgia Institute of Technology \\
  \And
  Zhenyu Liao \\
  Amazon \\
  \And
  Huayan Wang \\
  Cruise \\
  \And
  William Yang Wang \\
  UC Santa Barbara \\
  \And
  Qinxun Bai \\
  Horizon Robotics \\
}

\begin{document}
\maketitle

\begin{abstract}
  Improving the sample efficiency of reinforcement learning algorithms requires effective exploration. Following the principle of \textit{optimism in the face of uncertainty} (OFU), we train a separate exploration policy to maximize the approximate upper confidence bound of the critics in an off-policy actor-critic framework. However, this introduces extra differences between the replay buffer and the target policy regarding their stationary state-action distributions. To mitigate the off-policy-ness, we adapt the recently introduced DICE framework to learn a distribution correction ratio for off-policy RL training. In particular, we correct the training distribution for both policies and critics. Empirically, we evaluate our proposed method in several challenging continuous control tasks and show superior performance compared to state-of-the-art methods. We also conduct extensive ablation studies to demonstrate the effectiveness and rationality of the proposed method.
\end{abstract}

\section{Introduction}
% \jerry{(the last sentence needs some rewording, our method cannot avoid large amount of environment interactions using a premature policy either)}
% While these issues %can be addressed 
% have been studied from a variety of perspectives, we mainly focus on improving the exploration efficiency while stabilizing the training in an off-policy actor-critic framework. In particular, we maintain a replay buffer to store previously collected transitions, and train the policy and critics by sampling data from it. 

Incorporating high capacity function approximators (e.g., neural network) with actor-critic algorithms has shown great promise in many areas~\citep{mnih2013playing, silver2016mastering,kalashnikov2018qt,shani2005mdp,li2019multi}. However, its training instability and poor sample efficiency hinder its wider deployment in the real world. While these issues have been studied under different settings, one of the most powerful and practical frameworks is the off-policy actor-critic algorithm ~\citep{degris2012off,SAC,TD3,ddpg}, where a replay buffer is maintained to store previously collected transitions, and both the policy and the critic are trained by sampling transitions from it.
There are two major obstacles that a successful off-policy actor-critic algorithm has to overcome. Firstly, a carefully-designed exploration strategy is required to learn efficiently.
Secondly, off-policy training is unstable when combing with function approximation and bootstrapping, known as the \textit{deadly triad}~\citep{sutton2018reinforcement, van2018deep}. 

To stabilize training, state-of-the-art (SOTA) methods, e.g., SAC~\citep{SAC} and TD3~\citep{TD3}, 
use the minimum value between a pair of critics as an approximate lower bound to alleviate the overestimation of $Q$-values. To facilitate exploration, TD3 injects Gaussian white noise to the deterministic action while SAC explicitly trains its policy to have a target entropy depending on the dimensionality of the action space. Given the success of the OFU principle~\citep{auer2002finite,jin2020provably}, OAC proposes to maintain a separate exploration policy by maximizing an approximate upper bound of the critics and empirically showing that such an exploration strategy leads to more efficient exploration. To limit the amount of \textit{off-policyness} and stabilize training, OAC further constrains the KL divergence between the exploration and target policies.

While the intuition of OAC to leverage an ``optimistic" exploration policy could be beneficial for the efficient exploration, we argue that an important aspect of off-policy training has not been taken into account in all the above methods, i.e., a faithful policy gradient computation requires samples from the on-policy distribution which is in general not satisfied by directly sampling from the relay buffer in off-policy training. This issue becomes even more severe when a separate exploration policy is adopted. As a trade-off, OAC imposes a KL constraint between the exploration and the target policy, limiting the potential of deploying a more effective exploration policy to execute informative actions spuriously underestimated by the target policy. However, it still suffers from the distribution shift during off-policy training of the target policy.

In this work, we propose directly applying a distribution correction scheme in off-policy RL where a separate optimistic exploration policy can be trained solely on executing the most informative actions without worrying about its ``dissimilarity" to the target policy. Estimating the distribution correction ratio, however, is non-trivial.
Traditional methods based on importance sampling suffer from high variance~\citep{precup2001off}. The recently proposed Distribution Correction Estimation (DICE) family~\citep{nachum2019dualdice,nachum2019algaedice,zhang2020gendice, zhang2020gradientdice, uehara2020minimax, bestdice} provides us with a way to estimate the distribution correction ratio in a low-variance manner by solving a linear program (LP).
However, these methods are designed for the \textit{off-policy evaluation} (OPE) task, where both the target policy and replay buffer are fixed. It is unclear how to incorporate them into the more sophisticated off-policy RL.
% In this paper, we investigate to stabilize a more aggressive exploration strategy by correcting the training distribution. 

In this work, we successfully integrate the DICE optimization scheme into the off-policy actor-critic framework and stabilize the DICE-related training process by modifying the original optimization problem. 
Compared with OAC, our distribution-corrected training procedure enables learning a more effective exploration policy without enforcing any distance constraint between the exploration and the target policy.
Empirically, we show that our methods achieve a higher sample efficiency in challenging continuous control tasks than SOTA off-policy RL algorithms, demonstrating the power of the proposed distribution-corrected off-policy RL framework in allowing for a more effective exploration strategy design. Moreover, we examine the quality of our learned correction ratio by using it to construct the dual estimator proposed in \citep{bestdice} to estimate the average per-step rewards (on-policy rewards) of the RL target policy. We highlight that this is harder than the standard OPE setting, as our target policy and replay buffer are changing during RL training. Surprisingly, the dual estimator gives a decent estimation by matching the on-policy rewards, demonstrating the effectiveness of our distribution correction scheme.

% previous methods by employing a specialized exploration policy. 
To the best of our knowledge, this is the first work to improve the exploration efficiency of RL algorithms by explicitly correcting the training distribution rather than imposing divergence constraints, opening up a direction of potentially more profound exploration strategy design.

\section{Preliminaries}\label{sec:prili}

\subsection{Reinforcement Learning}
We consider the standard Reinforcement Learning (RL) setup that maximizes the expected discounted cumulative rewards in an infinite horizon 
% \jerry{(Our experimental tasks are finite-horizon?)}
Markov Decision Process (MDP). Formally, a MDP is defined by a tuple ($\mathcal{S}$, $\mathcal{A}$, $R$, $T$, $\gamma$), with state space $\mathcal{S}$, action space $\mathcal{A}$, reward function $R(s, a)$, transition function $T(s' | s, a)$, and discount factor $\gamma\in (0, 1)$. At timestep $t$, the agent selects an action $a_t \sim \pi(\cdot|s_t)$ for current state $s_t$ by policy $\pi(\cdot|s_t)$, receives a reward $r_t:=R(s_t, a_t)$ and transits to the next state $s_{t+1}\sim T(\cdot | s_t, a_t)$. We define the action value function of policy $\pi$ as $Q^\pi(s, a) = \mathbb{E}_{\tau\sim\pi:s_0=s, a_0=a}[\Sigma^{\infty}_{t=0}\gamma^t r_t]$, where $\tau = (s_0, a_0, r_0, s_{1}, a_{1}, r_{1}, \ldots)$ is a trajectory generated by rolling out in the environment following $\pi$. Thus, the standard RL objective can be expressed as $J =  \mathbb{E}_{s_0\sim\rho_0(s), a\in \pi}[Q^\pi(s_0, a)]$, where $\rho_0$ is the initial state distribution. We further denote the discounted marginal state distribution of $\pi$ as $d^\pi(s)$ and the corresponding state-action marginal as $d^\pi(s, a) = d^\pi(s)\pi(a|s)$. 
\subsection{Soft Actor-Critic (SAC)}
% it is promising to apply value-based methods to update the critics, and the critic's objective can be expressed as minimizing the expected single-step \textit{bellman error}. 

% the objective of critic's update is by minimizing the expected single-step \textit{bellman error}. When combined with function approximators, its training can be unstable, and suffer from overestimation.

% To tackle large state and action space, function approximators (e.g. neural network) are used to model the both the policy and critic. To improve sample efficiency, modern actor-critic algorithms train the critics in off-policy fashion and maintain a replay buffer to store history transitions. By sampling transitions from replay buffer, the critics can be learned via temporal-difference (TD) learning. However, TD learning is known to be unstable and suffers from overestimation with function approximation. 

To encourage exploration, SAC~\citep{SAC} optimizes the maximum entropy objective $J = \mathbb{E}_{\tau\sim\pi:s_0\sim\rho_0(s)}[\Sigma^{\infty}_{t=0}\gamma^t (r_t + \alpha \mathcal{H}(\cdot|\pi(s_t))]$, where $\mathcal{H}$ is the entropy function and $\alpha \in \mathbb{R}^+$. The resulting target policy $\pi$ has higher entropy, thus generating diverse actions during rollout. %exploration.

SAC learns the $Q$ value function as critics from the collected experience and uses it to guide the policy (actor) update. 
To stabilize training, SAC estimates lower confidence bound of the true $Q$ value by maintaining %double 
a pair of Q networks, $\hat{Q}_1$ and $\hat{Q}_2$, and taking their minimum, i.e., 
$\hat{Q}_{\mathrm{LB}}(s, a)=\min_{i\in\{1,2\}} \hat{Q}_i(s, a)$. The two Q networks are structured identically and trained in the same way but initialized differently.
The $Q$ value function is trained via value-based methods. The objective function minimizes the expected single-step \textit{bellman error}. It further keeps $\hat{Q}'_1$ and $\hat{Q}'_2$ as the exponential moving average of $\hat{Q}_1$ and $\hat{Q}_2$ to provide a stable training target. To this end, the target $Q$ value is computed by %the $Q$ function loss can be given by 
\begin{equation*}
\label{eq:q_target}
    \mathcal{B}^{\pi}\hat{Q}_{\mathrm{LB}}(s, a) = r - \alpha\gamma \log \pi(\cdot|s')
    + \gamma \min_{i\in\{1,2\}}\hat{Q}'_i\left(s', a'\right),
\end{equation*}
where $a' \sim \pi \left(\cdot \mid s'\right)$, and the training loss for $Q$ functions are
\begin{equation}
\label{eq:bias_q_loss}
    L^i_Q = \underset{(s, a, r, s')\sim D}{\mathbb{E}} \left(\hat{Q}_i(s, a) - \mathcal{B}^{\pi} \hat{Q}_{\mathrm{LB}}(s, a)\right)^2,
    i\in\{1, 2\}.
\end{equation}
As for the actor, SAC reparameterizes~\citep{kingma2013auto} the stochastic policy $\pi(\cdot|s)$ as $f_{\theta}(s, \varepsilon)$, where $\varepsilon \sim \mathcal{N}(0, I)$, and derives the policy gradient w.r.t its parameter $\theta$, which is given by
\begin{align}\label{eq:policy_grad}
    \nabla_{\theta} J^{\pi} = &\int_{s} d^\pi(s) \bigg(\int_{\varepsilon} \nabla_{\theta} \hat{Q}_{\mathrm{LB}}\left(s, f_{\theta}(s, \varepsilon)\right) \phi(\varepsilon) d \varepsilon + \alpha \int_{\varepsilon}-\nabla_{\theta} \log f_{\theta}(s, \varepsilon) \phi(\varepsilon) d \varepsilon \bigg)d s, 
\end{align}
where %$\varepsilon \sim \mathcal{N}(0, I)$ is a random variable, and 
$\phi(\varepsilon)$ denotes the density function of $\varepsilon$. 
% And $f_\theta$ is the reparameterization function so that $f_{\theta}(s, \varepsilon)$ returns the same distribution as $\pi(\cdot|s)$.
As we do not have access to the ground truth $d^\pi(s)$ during training, SAC approximates $d^\pi(s)$ with the state distribution $d^\mathcal{D}(s)$ of the replay buffer, and applies the following Monte-Carlo estimator of~\eqref{eq:policy_grad}:
\begin{equation}\label{eq:biased_estimator}
    % \nabla_{\theta} J^{\pi} & \approx 
    \nabla_{\theta} J^{\pi}
    \approx \sum_{s\in \mathcal{D}} \nabla_{\theta} \hat{Q}_{\mathrm{LB}}\left(s, f_{\theta}\left(s, \varepsilon\right)\right) - \alpha\nabla_{\theta} \log f_{\theta}\left(s, \varepsilon\right).
\end{equation}
However, when $d^\pi(s)$ exhibits large difference from $d^\mathcal{D}(s)$, the bias introduced by~\eqref{eq:biased_estimator} cannot be ignored. As we will show, correcting such a bias leads to large performance gain.

\subsection{Optimistic Actor-Critic (OAC)}
% \jiachen{TODO: 

% 1. Detail the epistemic uncertainty estimation

% 2. Talk about the target policy training, introduce $\beta_{LB}$

% 3. Determine if we need to go into the details on pessimistic underexplore \& directional uniformedness. 
% }

% Previous methods SAC and TD3 directly use the target Gaussian policy $\pi_T$ to explore the environment. The OAC paper identified that such an exploration strategy is inefficient, due to pessimistic under-exploration and the directional uniformedness of the Gaussian policy $\pi_T$. \jiachen{TODO: Need to add details here}. We refer interested readers to the original OAC paper for more details.

% As $\pi_T$ is trained to maximize the lower bound of a learned critic, it is prone to be trapped at the maximum of the lower bound. However, when the critic is inaccurate, the policy is hard to escape. 
OAC \citep{OAC} is built on top of SAC \citep{SAC}, and proposes to train a separate optimistic exploration policy $\pi_E$ in addition to the target policy $\pi_T$ for deployment. $\pi_E$ is designed to maximize an approximate upper confidence bound $\hat{Q}_{\mathrm{UB}}$ of the true $Q$ value, %besides 
while $\pi_T$ is trained to maximize an approximate lower bound $\hat{Q}_{\mathrm{LB}}$. 
%(\zhenyu{divide into two sentences?}). 
OAC shows that exploration using $\pi_E$ avoids \textit{pessimistic under-exploration} and \textit{directional uniformedness} \citep{OAC}, achieving higher sample efficiency over SAC.

% Besides having the policy $\pi_T$, OAC maintains a specific exploration policy $\pi_E$ to collect new transitions. Following the principle of \textit{optimism in the face of uncertainty} $\pi_E = \mathcal{N}(\mu_E, \sigma_E)$ is trained to maximize an upper confidence bound of the true $Q$ value. 

To obtain the approximate upper bound $\hat{Q}_{\mathrm{UB}}$, OAC 
% first estimates the epistemic uncertainty $\sigma_Q$ of the true $Q$ by modelling it as a Gaussian distribution. Therefore, its mean belief $\mu_Q$ and standard deviation $\sigma_Q$ can be given by
first computes the mean and variance of the $Q$ estimates from $\hat{Q}_1$ and $\hat{Q}_2$,
\begin{equation}\small
    \mu_Q = \frac{1}{2}(\hat{Q}_1 + \hat{Q}_2), \quad \sigma_Q = \sqrt{\sum_{i \in\{1,2\}} \frac{1}{2}\left(\hat{Q}_i-\mu_{Q}\right)^{2}}
    = \frac{1}{2}|\hat{Q}_1 - \hat{Q}_2|,
\end{equation}
% Thus, $\hat{Q}_{\mathrm{UB}}$ is defined as 
then define $\hat{Q}_{\mathrm{UB}} = \mu_Q + \beta_{\mathrm{UB}}\sigma_Q$, where $\beta_{\mathrm{UB}}\in \mathbb{R}^+$ controls the level of optimism. Note that the previous approximate lower bound $\hat{Q}_{\mathrm{LB}}(s, a)=\min_{i\in\{1,2\}} \hat{Q}_i(s, a)$
% $\hat{Q}_{\mathrm{LB}}$ 
can be expressed as $\hat{Q}_{\mathrm{LB}} = \mu_Q - \beta_{\mathrm{LB}}\sigma_Q$ with $\beta_{\mathrm{LB}} = 1$.

OAC derives a closed-form solution of the Gaussian exploration policy $\pi_E= (\mu_E, \Sigma_E)$, by solving an optimization problem that maximizes the expectation (w.r.t.\ $\pi_E$) of a linear approximation of $\hat{Q}_{\mathrm{UB}}$ %while constraining 
and constraints the maximum KL Divergence between $\pi_E$ and $\pi_T$. 
% The final solution turns out to be \textit{directionally informed}, i.e., $\mu_E$ shifts from the mean of the target policy $\mu_T$, while having $\sigma_E = \sigma_T$.
Given the trained target policy $\pi_T=(\mu_T, \Sigma_T)$, the closed-form solution of $\pi_E$ ends up having a $\mu_E$ from a \textit{directionally informed} shifts of $\mu_T$, and a $\Sigma_E$ equal to $\Sigma_T$.

% \begin{subequations}\label{prob:oac}
% \begin{align}
%     \mu_E, \sigma_E = \underset{\mu, \Sigma }{\arg \max }  \quad & \mathbb{E}_{a \sim \mathcal{N}(\mu, \Sigma)}\left[\bar{Q}_{\mathrm{UB}}(s, a)\right] \\
%     \textrm{s.t.} \quad & \mathrm{KL}\left(\mathcal{N}(\mu, \Sigma), \mathcal{N}\left(\mu_{T}, \Sigma_{T}\right)\right) \leq \delta \label{con:kl}\\
%     & \bar{Q}_{\mathrm{UB}}(s, a)=a^{\top}\left[\nabla_{a} \hat{Q}_{\mathrm{UB}}(s, a)\right]_{a=\mu_{T}}+\mathrm{const}\label{con:linear_approx}
% \end{align}
% \end{subequations}
Although the KL constraint enables a closed-form solution of $\pi_E$ and stabilizes the training, we argue that such a constraint 
% sacrifices $\pi_E$'s freedom to 
limits the potential of $\pi_E$
in executing more informative actions that may better correct the spurious estimates of critics, as $\pi_E$ is constrained to generate actions that are %close to 
not far from those generated by $\pi_T$
which is trained conservatively from the critics. 
In our framework, we remove the KL constraint 
to release the full potential of an optimistic exploration policy
% and directly train $\pi_E$ to optimize $\hat{Q}_{\mathrm{UB}}$, while maintaining the 
while addressing the training stability by explicitly correcting the biased gradient estimator in policy training.

\subsection{Distribution Correction Estimation}
% \jiachen{TODO: 

% 1. Introduce both the dual and primal problems.

% 2. Discuss the derivation of the optimization problem.
% }

Distribution Correction Estimation (DICE) family ~\citep{nachum2019dualdice, zhang2020gendice, zhang2020gradientdice, uehara2020minimax, nachum2020reinforcement, nachum2019algaedice, bestdice} is proposed to solve the OPE task,
which seeks an estimator of the policy value, i.e., the normalized expected per-step reward,
% It aims to estimate the policy value, defined as normalized expected per-step reward, 
from a static replay buffer $\mathcal{D}$.
DICE obtains an unbiased estimator %seeking to 
by estimating the distribution correction ratio $\zeta^*(s, a) = \frac{d^{\pi}(s, a)}{d^\mathcal{D}(s, a)}$, where $d^\mathcal{D}(s, a)$ is the state-action distribution of $\mathcal{D}$.

We adapt the DICE optimization scheme to the off-policy actor-critic framework, and use it to estimate the distribution correction ratio for our training.

\section{Optimistic Exploration with Explicit Distribution Correction}

% \begin{figure}[t]
%     \centering
%     \includegraphics[width=0.45\textwidth]{imgs/pipeline.pdf}
%     \caption{Overall pipeline of our off-policy actor critic framework}
%     \label{fig:pipeline}
% \end{figure}

\begin{figure*}[t!]
    \centering
    \includegraphics[width=0.70\textwidth]{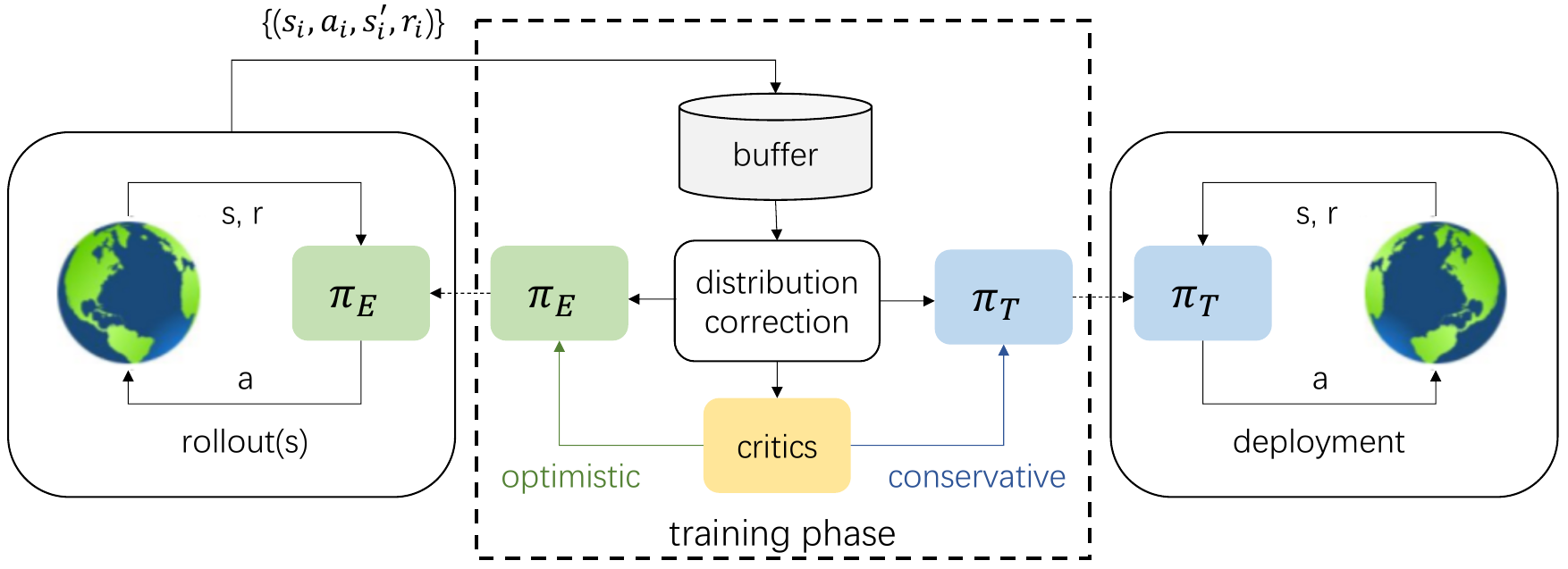}
    \caption{Overview of our off-policy actor critic framework}
    \label{fig:overview}
\end{figure*}

% Inspired by 
Following the philosophy of OAC and recent advances in the off-policy RL \citep{lee2021sunrise}, we propose to simultaneously %optimize 
train a conservative target policy $\pi_T$ for deployment, and %train 
an optimistic exploration policy $\pi_E$ %to execute 
for executing the most informative actions w.r.t the current $Q$ value estimation. 
Figure~\ref{fig:overview} provides an overview of our framework.
% By explicitly correcting 
We apply explicit distribution corrections to train both policies and critics, and no constraint is imposed on $\pi_E$ in terms of its divergence from $\pi_T$. 
% The overall pipeline is shown in \autoref{fig:pipeline}.

% In this section, 
In the rest of this section, we first introduce our off-policy training objectives and theoretically show that it forms an unbiased policy gradient estimator for $\pi_T$ with the distribution correction ratio. 
We then elaborate on how to estimate the distribution correction ratio by adapting the DICE optimization scheme.
% Finally, we sum up our algorithms and provide pseudo-codes.
We close this section by summarizing our algorithm with pseudo-codes.

\subsection{Off-policy Training Objectives}
% We parameterize 
$\pi_T$ and $\pi_E$ are parameterized by $\theta_T$ and $\theta_E$ respectively. We start from the training objective for the target policy $\pi_T$. To prevent from overestimation, we follow previous work~\citep{SAC,TD3,OAC} in using the approximate lower bound $\hat{Q}_{\mathrm{LB}}$ as the critic,
and correct the sample distribution with $\frac{d^{\pi_T}(s, a)}{d^\mathcal{D}(s, a)}$,
\begin{equation}\label{obj:cor_target}
    \hat{J}^{\pi_T}= \sum_{(s, a)\in \mathcal{D}} \frac{d^{\pi_T}(s, a)}{d^\mathcal{D}(s, a)}\big( \hat{Q}_{\mathrm{LB}}\left(s, f_{\theta_T}\left(s, \varepsilon\right)\right)
    - \alpha \log f_{\theta_T}\left(s, \varepsilon\right)\big).
\end{equation}
The following proposition (proved in the \autoref{sec:proof_3p1}) shows that the gradient of training objective~\eqref{obj:cor_target} provides an unbiased estimation of the policy gradient~\eqref{eq:policy_grad}. This contrasts with previous off-policy actor-critic algorithms that directly apply the biased estimator~\eqref{eq:biased_estimator} in training the target policy.
\begin{prop}\label{prop:correct_dist}
$\nabla_{\theta_T}\hat{J}^{\pi_T}$ gives an unbiased estimate of the policy gradient \eqref{eq:policy_grad} with $\pi = \pi_T$.
% \begin{equation}\label{unbiased_target_est}
%     \nabla_{\theta_T}\hat{J}^{\pi_T} = \sum_{(s, a)\in \mathcal{D}}\frac{d^{\pi_T}(s, a)}{d^\mathcal{D}(s, a)} \nabla_{\theta_T} \left(\hat{Q}_{\mathrm{LB}}\left(s, f_{\theta_T}\left(s, \varepsilon\right)\right) - \alpha \log f_{\theta_T}\left(s, \varepsilon\right)\right)
% \end{equation}
\end{prop}
% We defer the proof of Proposition~\ref{prop:correct_dist} to the Appendix.
% \begin{proof}
% $\nabla_{\theta_T}\hat{J}_{cor, \mathrm{LB}}$ is an unbiased estimator for
% \begin{align*}
%      &\quad\int_{s}\int_{a} d^{\mathcal{D}}(s, a) \frac{d^{\pi_T}(s, a)}{d^\mathcal{D}(s, a)} \left(\int_{\varepsilon} \nabla_{\theta_T} \left(\hat{Q}_{\mathrm{LB}}\left(s, f_{\theta_T}(s, \varepsilon)\right) - \alpha\log f_{\theta_T}(s, \varepsilon)\right) \phi(\varepsilon) d \varepsilon \right)d a d s  \\
%   &  = \int_{s} d^{\pi_T}(s) \left(\int_{\varepsilon} \nabla_{\theta_T} \left(\hat{Q}_{\mathrm{LB}}\left(s, f_{\theta_T}(s, \varepsilon)\right) - \alpha\log f_{\theta_T}(s, \varepsilon)\right) \phi(\varepsilon) d \varepsilon \right) d s= \nabla_\theta J^{\pi_T}
% \end{align*}
% \end{proof}

Regarding the exploration policy $\pi_E$, there are two considerations. Firstly, it should explore optimistically w.r.t. the estimated $Q$ values, in order to collect informative experience to efficiently correct spurious $Q$ estimates. Therefore, we follow OAC to use the approximate upper bound $\hat{Q}_{\UB}$ as critic in the objective.
Secondly, 
% note that 
the ultimate goal of $\pi_E$ and better $Q$ value estimation is to facilitate a more accurate gradient estimate for $\pi_T$.
% through collecting informative experience to correct the spurious $Q$ value estimation.
% it should serve to reduce the uncertainty of $\pi_T$'s gradient estimate. 
Therefore, the training sample distribution for 
% that construct 
$\pi_E$'s objective should 
% come from the stationary state distribution of 
be consistent with that for
$\pi_T$'s. As a result, we apply the same correction ratio as in~\eqref{obj:cor_target} to get the following objective for $\pi_E$, 
% objective below
\begin{equation}\label{obj:cor_explore}
    \hat{J}^{\pi_E} = \sum_{(s, a)\in \mathcal{D}} \frac{d^{\pi_T}(s, a)}{d^\mathcal{D}(s, a)}\big( \hat{Q}_{\mathrm{UB}}\left(s, f_{\theta_E}\left(s, \varepsilon\right)\right)
    - \alpha \log f_{\theta_E}\left(s, \varepsilon\right)\big),
\end{equation}
again, this is different from all existing works that adopt a separate exploration policy.

As for training the critics, 
while the best sample distribution for performing off-policy Q learning is still an open problem \citep{DISCOR},
there exists some %both 
empirical work\citep{LFIW,nachum2019algaedice} 
% and theoretical work \citep{sutton2018reinforcement,nachum2018smoothed,silver2014deterministic} 
showing that using samples from the 
% on-policy distribution 
same distribution as that for policy training
benefits the training of $Q$ functions. Therefore, we also apply the same correction ratio to $Q$ function objectives~\eqref{eq:bias_q_loss} and obtain the corrected objectives: 
\begin{equation}
\label{obj:cor_q}
    \hat{L}^i_Q = \sum_{(s, a)\in \mathcal{D}}  \frac{d^{\pi_T}(s, a)}{d^\mathcal{D}(s, a)} \left(\hat{Q}_i(s, a) - \mathcal{B}^{\pi}\hat{Q}_{\LB}(s, a)\right)^2
\end{equation}

\subsection{Estimating Distribution Correction Ratio via DICE}\label{sec:correct_for_rl}
We estimate the distribution correction ratio $\frac{d^{\pi_T}(s, a)}{d^\mathcal{D}(s, a)}$ by adapting the DICE optimization scheme~\citep{bestdice}. \cite{bestdice} shows that existing DICE estimators can be formulated as the following min-max distributional linear program with different regularizations, %as below.
\begin{align}\label{eq:prob_dice}
    \max_{\zeta \geq 0} \min _{\nu, \lambda}\,\, L_{D}(\zeta, \nu, \lambda) := & (1-\gamma) \cdot \mathbb{E}_{a_{0} \sim \pi_T\left(s_{0}\right) \atop s_{0}\in \rho_0(s)}\left[\nu\left(s_{0}, a_{0}\right)\right] \nonumber \\
    & +\lambda\cdot \left(1 - \mathbb{E}_{\left(s, a, r, s^{\prime}\right) \sim d^{\mathcal{D}}}\left[\zeta(s, a)\right]\right) - \alpha_{\zeta} \cdot \mathbb{E}_{(s, a) \sim d^{\mathcal{D}}}\left[g_{2}(\zeta(s, a))\right]\nonumber\\
    & +\mathbb{E}_{\left(s, a, r, s^{\prime}\right) \sim d^{\mathcal{D}}}\left[\zeta(s, a) \cdot\left(\mathcal{B}^{\pi_T}\nu(s, a)-\nu(s, a)\right)\right] \nonumber\\
    & +\alpha_{\nu} \cdot \mathbb{E}_{(s, a) \sim d^{\mathcal{D}}}\left[g_{1}(\nu(s, a))\right],
\end{align}
where $g_1$ and $g_2$ are some convex and lower-semicontinuous regularization functions, $\alpha_\nu$ and $\alpha_\zeta$ control the corresponding regularization strength. By slightly abusing the notation, we define $\mathcal{B}^{\pi_T}\nu(s, a) = \alpha_{R} \cdot R(s, a)+\gamma \nu\left(s^{\prime}, a^{\prime}\right), a' \sim \pi_T(s')$ with $\alpha_R$ scaling the reward. Solving the optimization problem~\eqref{eq:prob_dice} for $\zeta$ %above 
gives us $\zeta^{*}(s, a) = \frac{d^{\pi_T}(s, a)}{d^\mathcal{D}(s, a)}$, which is exactly
the distribution correction ratio.
Note that the formulation supports both discounted and undiscounted return cases. 

We now provide intuitions for \eqref{eq:prob_dice}. For simplicity, we refer to RHS as the RHS of \eqref{eq:prob_dice}. As $\zeta^{*}(s, a)$ is the distribution correction ratio, $\lambda$ is the Lagrangian coefficient to ensure the expectation of $\zeta^{*}(s, a)$ to be $1$. $\nu$ can be viewed as a synonym to the $Q$, with the first line of the RHS giving the normalized expected per-step reward of policy $\pi_T$. Line 3 and 4 together can be viewed as the expected single-step \textit{bellman error}, which becomes clear if referring to the Fenchel Duality and setting $g_2(\zeta) = \frac{1}{2}\zeta^2$. We refer interested readers to the original paper \citep{bestdice} for more details.

% By parameterizing $\nu$ and $\zeta$ using function approximators, OPE algorithms \citep{nachum2019dualdice,zhang2020gendice,nachum2019algaedice} directly solve for problem \eqref{eq:prob_dice} via gradient descent. (\zhenyu{The first sentence I would personally prefer to point out the fundamental issue for OPE. This sentence seems a little bit unrelated to our theme here. Gradient descent itself is not a problem 
% anyway. }) However, implementing such algorithms into 
Directly applying the above DICE algorithms to
the off-policy RL settings cast great optimization difficulties. While OPE assumes a fixed target policy and a static replay buffer with sufficient state-action space coverage \citep{nachum2019dualdice}, in RL, both the target policy and the replay buffer change during training. %in RL. 
We thus make three modifications to~\eqref{eq:prob_dice} to overcome the challenges.

Our first modification to~\eqref{eq:prob_dice} is removing the $\mathbb{E}_{ s_{0}\in \rho_0(s), a_{0} \sim \pi\left(s_{0}\right)}\left[\nu\left(s_{0}, a_{0}\right)\right]$ term from the objective. 
In the initial training stage, the state-action pairs stored in the replay buffer only achieve little coverage of the entire state-action space, especially for high dimensional continuous control tasks. %And 
As a result, the target policy $\pi_T$ may sample some $a_0$ that is very different from any collected experience in the replay buffer.
% for given initial state $s_0$. 
% Consequently, 
$\nu(s_0, a_0)$ will then be minimized pathologically low and thus incur a huge error. Note that similar phenomena also appear in the offline RL~\citep{fujimoto2019off, kumar2019stabilizing}. 
Our second modification is to add an absolute value operator $|\cdot|$ to the single step bellman error $\mathcal{B}^{\pi}\nu(s, a) - \nu(s, a)$. We empirically find that it increases the training stability. The intuition comes from considering solving for $L_{\zeta}$ defined by~\eqref{loss:zeta}, since when $g_2$ is square function then the absolute value automatically helps the $\zeta$ function lie on the positive region without any hard projection.

Our third modification to the original problem~\eqref{eq:prob_dice} is to set both $\alpha_\nu$ and $\alpha_\zeta$ greater than zero.
This is different from existing DICE schemes~\citep{bestdice} that set either $\alpha_\nu$ or $\alpha_\zeta$ to be zero.
Though such a setting introduces bias and has never been applied in existing DICE algorithms, it dramatically stabilizes the training in practice.
Moreover, adding two regularizations on a min-max program has been widely applied for faster convergent in optimization literature \citep{allenzhu_et_al:LIPIcs:2016:6332}. 

To this end, we obtain as follows the loss functions for each of $\nu$, $\zeta$, and $\lambda$ respectively,
\begin{align}
    L_{\lambda} & = \lambda \cdot \left(1 - \mathbb{E}_{\left(s, a, r, s^{\prime}\right) \sim d^{\mathcal{D}}}[\zeta(s, a)]\right),
    \label{loss:lambda}\\
    L_{\nu} & = \mathbb{E}_{\left(s, a, r, s^{\prime}\right) \sim d^{\mathcal{D}}}\left[\zeta(s, a) \cdot\left|\mathcal{B}^{\pi_T}\nu(s, a)-\nu(s, a)\right|\right] + \alpha_{\nu} \cdot \mathbb{E}_{(s, a) \sim d^{\mathcal{D}}}\left[g_{1}(\nu(s, a))\right], \label{loss:nu}\\
    L_{\zeta} & = \alpha_{\zeta} \cdot \mathbb{E}_{(s, a) \sim d^{\mathcal{D}}}\left[g_{2}(\zeta(s, a))\right] -\mathbb{E}_{\left(s, a, r, s^{\prime}\right) \sim d^{\mathcal{D}}}\left[\zeta(s, a) \cdot\left(\left|\mathcal{B}^{\pi_T}\nu(s, a)-\nu(s, a)\right| - \lambda\right)\right]. \label{loss:zeta}
\end{align}
Both $\zeta$ and $\nu$ are parameterized by neural networks with parameters $\eta_{\zeta}$ and $\eta_{\nu}$ respectively. To stabilize the training, we maintain a $\nu'$ as the exponential moving average of $\nu$ to compute $\mathcal{B}^{\pi_T}\nu(s, a)$.

When applying the learned $\zeta(s, a)$ as a surrogate of $ \frac{d^{\pi_T}(s, a)}{d^\mathcal{D}(s, a)}$ in the RL training objectives, we need to tackle the issue of finite samples since our training is in minibatch.
% \zhenyu{may be one more sentence explaining why minibatch requires tempareture? } 
Inspired by \citep{LFIW}, we apply self-normalization to $\zeta(s, a)$ with temperature $T$, and use the $\tilde{\zeta}(s, a)$ defined below to construct the RL objectives \eqref{obj:cor_target}, \eqref{obj:cor_explore} and \eqref{obj:cor_q}.
\begin{equation}\label{eq:zeta}
    \tilde{\zeta}(s, a) = \frac{\zeta(s, a)^{1/T}}{\sum_{(s, a)\sim \mathcal{D}}\zeta(s, a)^{1/T}}.
\end{equation}
\subsection{Summary of Algorithm}
To this end, we summarize our overall algorithm in Algorithm~\ref{alg}. 
% Unlike OAC, we update both $\pi_T$ and $\pi_E$ via gradient descent. 
Like conventional off-policy actor-critic algorithms, we alternatively perform environment interactions following $\pi_E$ and update of all trainable parameters. 
In each training step, our algorithm first updates the $\nu$, $\zeta$, and $\lambda$ related to DICE, using SGD w.r.t. losses~\eqref{loss:nu},~\eqref{loss:zeta}, and~\eqref{loss:lambda} respectively, and computes the distribution correction ratio $\tilde{\zeta}$ from the updated $\zeta$. With $\tilde{\zeta}$, our algorithm then performs RL training to update $\pi_T$, $\pi_E$, $\hat{Q}_1$, and $\hat{Q}_2$. Finally, soft update of target critics $\hat{Q}'_1$, $\hat{Q}'_2$, and $\nu'$ is performed

\begin{algorithm}[t]
    \caption{Optimistic Exploration with Explicit Distribution Correction}\label{alg}
    {\bf Input}: target and explore policy parameters $\theta_T$, $\theta_E$; $\hat{Q}_1$, $\hat{Q}_2$ parameters $\omega_1$, $\omega_2$; $\hat{Q}'_1$, $\hat{Q}'_2$ parameters $\omega'_1$, $\omega'_2$; $\nu$, $\zeta$ parameters $\eta_\nu$, 
    $\eta_\zeta$, $\nu'$ parameters $\eta_{\nu'}$, and $\lambda$; replay buffer $\mathcal{D}$
    \begin{algorithmic}[1]
        \State $w'_{1} \leftarrow w_{1}$, $w'_{2} \leftarrow w_{2}$,  $\eta_{\nu'} \leftarrow \eta_{\nu}$,  $\mathcal{D}\leftarrow \emptyset$
        \Repeat
            \For{each environment step}
                \State Sample action $a_t\sim\pi_E(\cdot|s_t)$
                \State Observe $s_{t+1}\sim T(\cdot|s_t, a_T)$ and $r\left(\mathbf{s}_{t}, \mathbf{a}_{t}\right)$
                \State $\mathcal{D} \leftarrow \mathcal{D} \cup\left\{\left(\mathbf{s}_{t}, \mathbf{a}_{t}, r\left(\mathbf{s}_{t}, \mathbf{a}_{t}\right), \mathbf{s}_{t+1}\right)\right\}$
            \EndFor

            \For{each training step}
                \State Sample minibatch $B = \{ (s,a,r,s') \}$ from $\mathcal{D}$ 
                \State \# DICE training
                \State Update $\eta_\nu$, $\eta_\zeta$, $\lambda$ via $\nabla_{\eta_\nu} L_{\nu}$, $\nabla_{\eta_\zeta} L_{\zeta}$,  $\nabla_{\psi_\lambda} L_{\lambda}$
                % \State Update $\lambda$ with $\nabla_{\psi_\lambda} L_{\lambda}$
                \State Calculate $\tilde{\zeta}$ with \eqref{eq:zeta}
                \State \# RL training
                \State Update $\theta_T$ with $\nabla_{\theta_T}\hat{J}^{\pi_T}$
                \State Update $\theta_E$ with $\nabla_{\theta_E}\hat{J}^{\pi_E}$
                \State Update $w_{i}$ with $\nabla_{w_{i}}\hat{L}^{i}_Q$, $i\in\{ 1, 2\}$
                \State \# Soft update
                \State $w'_{i} \leftarrow \tau w_{i}+(1-\tau) w'_{i}$, $i\in\{ 1, 2\}$; $\eta_{\nu'} \leftarrow \tau \eta_{\nu} + (1-\tau) \eta_{\nu'}$
            \EndFor
        \Until{convergence}
    \end{algorithmic}
\end{algorithm}

\section{Experimental Results}
We validate the proposed algorithm from three perspectives. First, we compare our algorithm with SOTA off-policy actor-critic algorithms on challenging control tasks. 
Second, we examine the quality of our learned correction ratio by evaluating whether it helps to provide a reasonable estimate of the on-policy rewards.
% Second, we analyze whether our DICE optimization scheme provides a reasonable estimation of the on-policy rewards. 
Third, we ablate key components of our algorithms to evaluate their contributions to the overall performance.

We evaluate on 5 continuous control tasks from the MuJoCo \citep{todorov2012mujoco} benchmark and compare it with SAC \citep{SAC} and OAC \citep{OAC}. We further implement a strong baseline OAC + DICE by integrating our DICE correction into OAC, i.e., similar to \eqref{obj:cor_target} and \eqref{obj:cor_q}, using the learned distribution correction ratio to correct the sample distribution for both the target policy training and the critics training.

For each task, we run 5 different random seeds for each algorithm. We evaluate each algorithm every 1000 training steps. Our algorithm introduces additional hyper-parameters (HP), including $\beta_{\mathrm{LB}}$ and $\beta_{\mathrm{UB}}$ that control the extent of conservativeness and optimism respectively, the temperature $T$ from~\eqref{obj:cor_q}, and the set of hyper-parameters from the DICE optimization~\eqref{loss:lambda},~\eqref{loss:nu},~\eqref{loss:zeta}. We provide detailed hyper-parameter tuning procedures in the \autoref{sec:hp}. 
Note that the implementation of our algorithms are based on OAC's officially released codes\footnote{\url{https://github.com/microsoft/oac-explore}}, which themselves are based on the SAC implementation in RLkit \footnote{\url{https://github.com/rail-berkeley/rlkit}}. 
While we are aware of different implementations of SAC with performance variance on different tasks, 
given the difficulty in reproducing the best RL results in practice~\citep{engstrom2020implementation, pineau2021improving},
% We are aware that different implementations lead to a different performance in RL, and 
% we emphasize that we are making 
we have tried our best to make all comparisons fair and reproducible by implementing all methods and variants in a unified and credible codebase.
% fair comparison between our algorithms and all the other baselines, given the difficulty to reproduce the best RL results in practice \citep{engstrom2020implementation, pineau2021improving}.

% \jerry{Note that for the term
% $\mathcal{B}^{\pi_T}\nu(s, a) = \alpha_{R} \cdot R(s, a)+\gamma \nu\left(s^{\prime}, a^{\prime}\right), a' \sim \pi_T(s')$
% in~\eqref{loss:zeta},
% we set $\gamma < 1$, 
% % (\zhenyu{I think you said $\gamma = 1$ above. You need to provide the definition of  $\mathcal{B}^{\pi_T}\nu(s, a)$ here to avoid confusion })
% % in $\mathcal{B}^{\pi_T}\nu(s, a)$, 
% as we empirically find it improves training stability.}

\subsection{Evaluation on Continuous Control Tasks}
\label{sec:avg_return}

% \begin{figure*}[h]
%   \makebox[\textwidth]{
  
%   \begin{subfigure}{\appendixresultfigsize\paperwidth}
%     \includegraphics[width=\linewidth]{imgs/ub_lb_t_3p0_result_ant.png}
%   \end{subfigure}
  
%   \begin{subfigure}{\appendixresultfigsize\paperwidth}
%     \includegraphics[width=\linewidth]{imgs/ub_lb_t_3p0_result_humanoid.png}
%   \end{subfigure}
 
%   \begin{subfigure}{\appendixresultfigsize\paperwidth}
%     \includegraphics[width=\linewidth]{imgs/ub_lb_t_3p0_result_halfcheetah.png}
%   \end{subfigure}
%     }

%   \makebox[\textwidth]{
%   \begin{subfigure}{\appendixresultfigsize\paperwidth}
%     \includegraphics[width=\linewidth]{imgs/ub_lb_t_3p0_result_walker2d.png}
%   \end{subfigure}

%   \begin{subfigure}{\appendixresultfigsize\paperwidth}
%     \includegraphics[width=\linewidth]{imgs/ub_lb_t_3p0_result_hopper.png}
%   \end{subfigure}}

%   \caption{Ablation study on the $\beta_{\mathrm{UB}},\beta_{\mathrm{LB}}$ of our method on 4 MuJoco tasks. 
%   x-axis indicates the number of environment steps. 
%   y-axis indicates the total undiscounted return.
%   The shaded areas denote one standard deviation between 5 runs of different random seeds.}
% \label{fig:ablations-ub-lb}
% \end{figure*}

\begin{figure*}[t]
  \makebox[\textwidth]{
  
  \begin{subfigure}{\mainfigsize\paperwidth}
    \includegraphics[width=\linewidth]{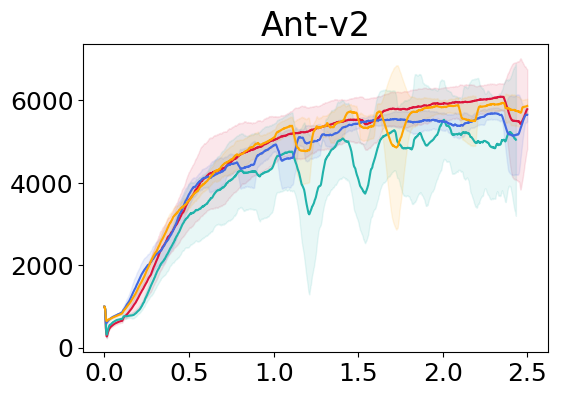}
  \end{subfigure}
  
  \begin{subfigure}{\mainfigsize\paperwidth}
    \includegraphics[width=\linewidth]{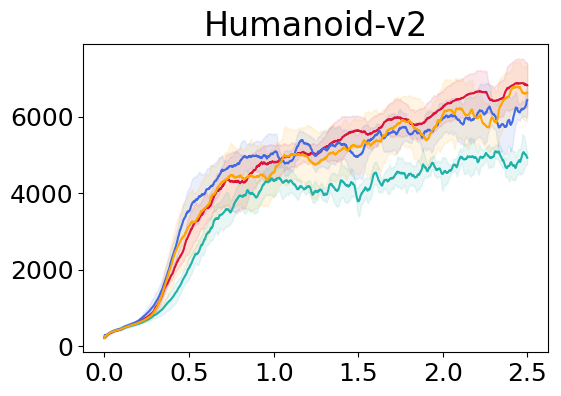}
  \end{subfigure}
    
  \begin{subfigure}{\mainfigsize\paperwidth}
    \includegraphics[width=\linewidth]{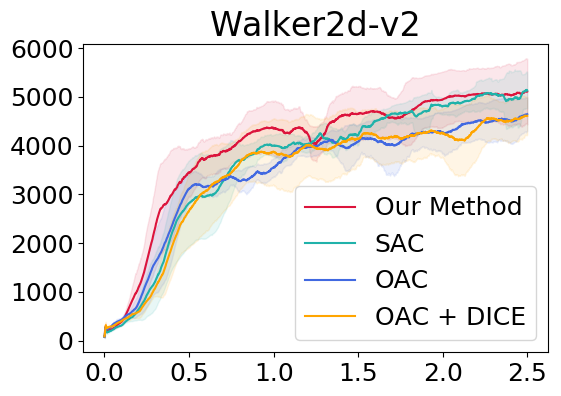}
  \end{subfigure}
  }
  
  \makebox[\textwidth]{
  \begin{subfigure}{\mainfigsize\paperwidth}
    \includegraphics[width=\linewidth]{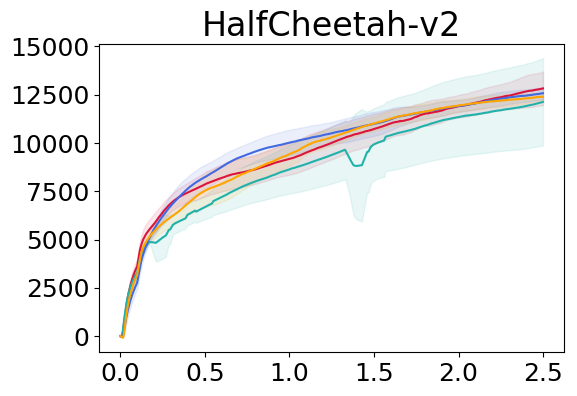}
  \end{subfigure}
  
  \begin{subfigure}{\mainfigsize\paperwidth}
    \includegraphics[width=\linewidth]{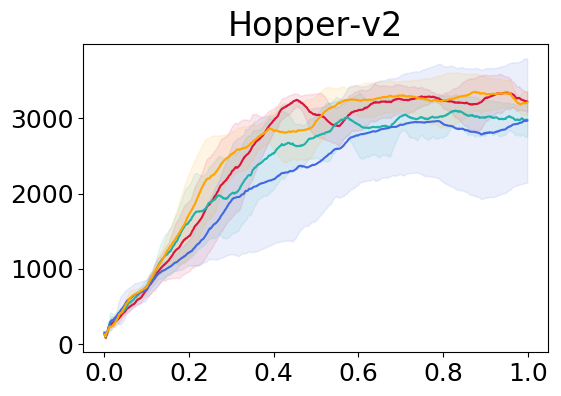}
  \end{subfigure}}
    
  \caption{Our methods versus SAC, OAC, and the strong baseline OAC + DICE on 5 MuJoco tasks. 
  x-axis indicates the number of environment steps. 
  y-axis indicates the total undiscounted return.
  The shaded areas denote one standard deviation between 5 runs of different random seeds.}
\label{fig:Ours_vs_SAC_OAC}
\end{figure*}

As shown in \autoref{fig:Ours_vs_SAC_OAC}, our method consistently outperforms both OAC and SAC on 4 out of the 5 tasks, while being on par with OAC on the less challenging HalfCheetah. 
The superior performance of our method vs OAC and SAC demonstrates the effectiveness of the proposed framework and the benefits of both the optimistic exploration and distribution correction for off-policy RL.
% Our exceptional performance on the more complicated Humanoid, Ant, and Walker2d demonstrates the effectiveness of our method and the benefits of doing explicit exploration while explicitly correcting the training distribution.

We would also like to highlight the performance of our specific baseline OAC + DICE.
% Next, We compare OAC + DICE with OAC. 
While it outperforms OAC by a large margin on Hopper, there are only slight improvements over OAC on Humanoid and Ant, even with no performance gain on Walker2d and HalfCheetah. These results demonstrate that the proposed distribution correction scheme at least does not harm the performance and often leads to improvements for off-policy actor-critic algorithms beyond ours. 
Moreover, the performance of our method vs. OAC + DICE on the more challenging Humanoid, Ant, and Walker2d demonstrates that our exploration policy design is superior to that of OAC when paired with the DICE distribution correction. 
% However, the improvement made by OAC + DICE over OAC is limited compared to our method's improvement upon its baseline without distribution correction, as in \autoref{fig:ablations}. 
This is not surprising. Due to the use of KL constraint in OAC, 
the actions generated by $\pi_E$ can not differ much from those generated by $\pi_T$ for most of the states. 
As a result, the potential of $\pi_E$ in collecting the most informative experience for correcting the critics has been largely constrained. 
Our method, on the other hand, does not impose any constraints on the distance between $\pi_E$ and $\pi_T$, and therefore learns a more effective exploration policy for collecting the informative experience for off-policy training.

Similarly, we implement SAC + DICE by integrating our DICE correction ratio into SAC without any HP tuning. As shown in \autoref{fig:sac-dice}, SAC + DICE outperforms SAC on the challenging Humanoid and Walker2d by a clear margin while being on par with SAC on the Ant. There is also a slight performance gain on the Hopper. It is worth noting that integrating our learned DICE correction ratio into \textbf{NO DICE} (a variant of ours), OAC, and SAC all results in a performance gain. We thus highlight that our proposed DICE correction scheme
effectively benefits general off-policy actor-critic training.
\begin{figure*}[t]
  \makebox[\textwidth]{
  
  \begin{subfigure}{\sacfigsize\paperwidth}
    \includegraphics[width=\linewidth]{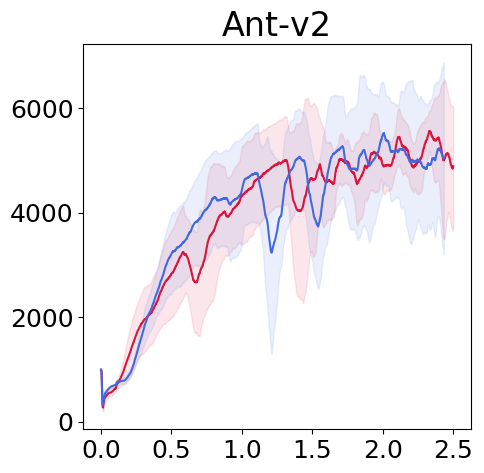}
  \end{subfigure}
  
  \begin{subfigure}{\sacfigsize\paperwidth}
    \includegraphics[width=\linewidth]{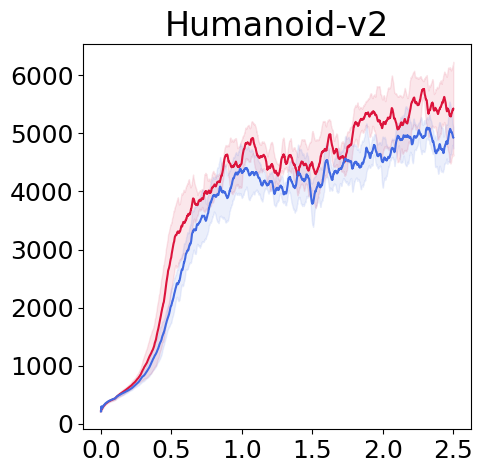}
  \end{subfigure}

  \begin{subfigure}{\sacfigsize\paperwidth}
    \includegraphics[width=\linewidth]{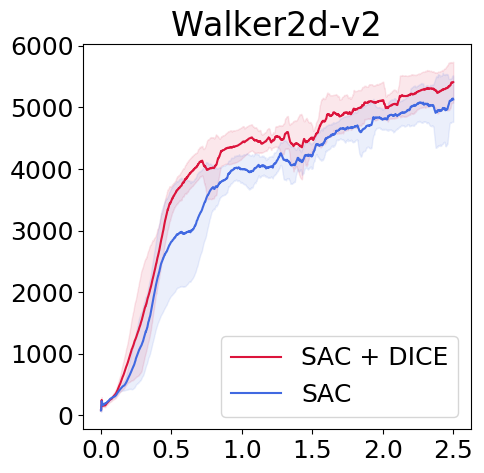}
  \end{subfigure}

  \begin{subfigure}{\sacfigsize\paperwidth}
    \includegraphics[width=\linewidth]{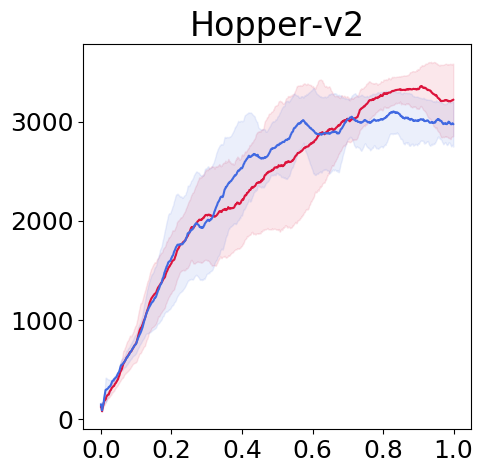}
  \end{subfigure}}

  \caption{Comparing SAC + DICE vs. SAC on 4 Mujoco environments. 
  x-axis indicates the number of environment steps. 
  y-axis indicates the total undiscounted return.
  The shaded areas denote one standard deviation between 5 runs of different random seeds.}
\label{fig:sac-dice}
\end{figure*}
% Note that our method still outperforms such a strong baseline OAC + DICE somewhat on Humanoid, Ant, Hopper and especially on Walker2d. We attribute the success to our exploration policy, which can be trained solely to execute the most informative action to correct flawed Q estimates, as the learned correction ratio will take care about the training distribution. 

\subsection{Examine the Correction Ratio}

\begin{figure*}[t]
  \makebox[\textwidth]{
  
  \begin{subfigure}{\resultfigsize\paperwidth}
    \includegraphics[width=\linewidth]{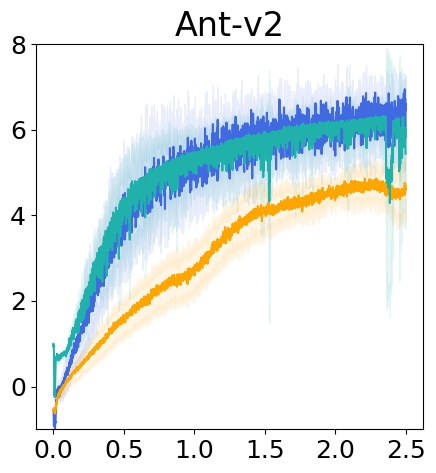}
  \end{subfigure}
  
  \begin{subfigure}{\resultfigsize\paperwidth}
    \includegraphics[width=\linewidth]{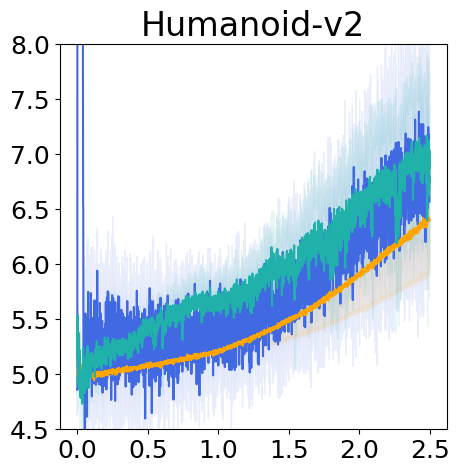}
  \end{subfigure}
    
  \begin{subfigure}{\resultfigsize\paperwidth}
    \includegraphics[width=\linewidth]{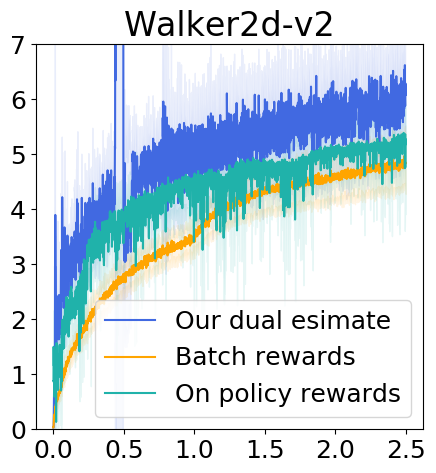}
  \end{subfigure}

  \begin{subfigure}{\resultfigsize\paperwidth}
    \includegraphics[width=\linewidth]{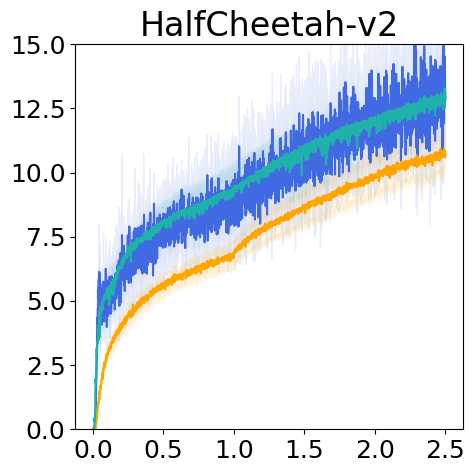}
  \end{subfigure}
  
  \begin{subfigure}{\resultfigsize\paperwidth}
    \includegraphics[width=\linewidth]{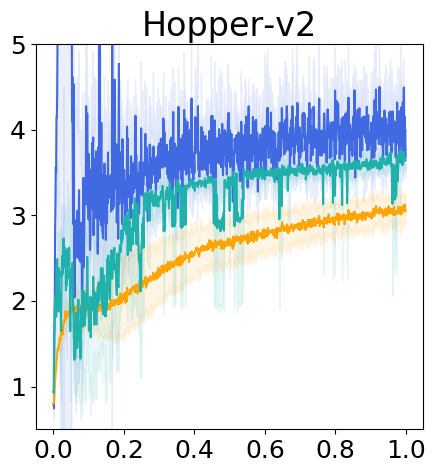}
  \end{subfigure}}
%   \centering{
%       \small{\color{blue}\textbf{---} }: Our dual estimate \qquad  {\color{red}\textbf{---} }: Batch rewards \qquad{\color{green}\textbf{---} }: On-policy Rewards
%   }

  \caption{Reward estimation results of our method (Our dual estimate), a naive baseline (Batch rewards), and a reference (On-policy rewards) on 5 Mujoco environments. 
  x-axis is the number of environment steps. 
  y-axis is the reward value.
  The shaded areas denote one standard deviation between 5 runs of different random seeds. }
\label{fig:rewards}
\end{figure*}

% \jiachen{To examine the quality of the learned distribution correction ratio, we use it to construct the dual estimator proposed in \citep{bestdice} to concurrently estimate the average per-step rewards (on-policy rewards) of the RL-learned target policy $\pi_T$. We highlight that this is harder than the standard OPE setting, as both $\pi_T$ and replay buffer are changing during RL training.}

To examine the quality of the learned distribution correction ratio, we propose to tackle a novel OPE setting that is more challenging than the traditional counterpart, where both the target policy and the replay buffer are changing. Following our Algorithm~\ref{alg} and the same experiment settings in Section~\ref{sec:avg_return}, 
we now use the following dual estimator defined in \citep{bestdice} to estimate the on-policy average per-step rewards $r^{\pi_T}_{on}$ (on-policy rewards) of $\pi_T$,
\begin{equation}
    \label{eq:dual_est}
    \hat{\mu}_{\zeta}(\pi_T) = \sum_{(s, a, r)\in\mathcal{D}}\zeta(s, a) r,
\end{equation}
where $\zeta$ is the distribution correction ratio estimated in Algorithm~\ref{alg}. We use the metric 
$|\hat{\mu}_{\zeta}(\pi_T) - r^{\pi_T}_{on}|$
% \begin{equation}\label{eq:ope_metric}
%     |\hat{\mu}_{\zeta}(\pi_T) - r^{\pi_T}_{on}|
% \end{equation}
to evaluate the OPE performance, which directly reflects the quality of $\zeta(s, a)$.
If $|\hat{\mu}_{\zeta}(\pi) - r^{\pi_T}_{on}| < |r_{\mathcal{D}} - r^{\pi_T}_{on}|$, where $r_{\mathcal{D}}$ is the mean reward in the current replay buffer $\mathcal{D}$, correcting the training distribution of RL using the learned $\zeta$ can at least provide positive momentum to the overall training.

Figure~\ref{fig:rewards} shows the performance of two OPE methods, i.e.,  ``Our dual estimate"~\eqref{eq:dual_est} and the naive ``Batch rewards" baseline, both are estimated using a batch of transitions. We also plot the ``on policy rewards" $r^{\pi_T}_{on}$ as a reference which averages over the rewards collected by the target policy during evaluation. As shown in \autoref{fig:rewards}, our OPE performance is surprisingly decent. The $\hat{\mu}_{\zeta}(\pi)$ learned by our method nicely matches the on-policy reward estimate on Humanoid, Ant, and HalfCheetah, while staying close to the on-policy reward estimate on Hopper and Walker2d. These experimental results demonstrate the effectiveness of our adaptation of the DICE scheme from standard OPE to the off-policy RL setting, which also explains the success of our proposed algorithm in experiments of Section~\ref{sec:avg_return} where the power of this adapted DICE scheme has been leveraged.

\subsection{Ablation Study}

\begin{figure*}[t]
  \makebox[\textwidth]{
  
  \begin{subfigure}{\mainfigsize\paperwidth}
    \includegraphics[width=\linewidth]{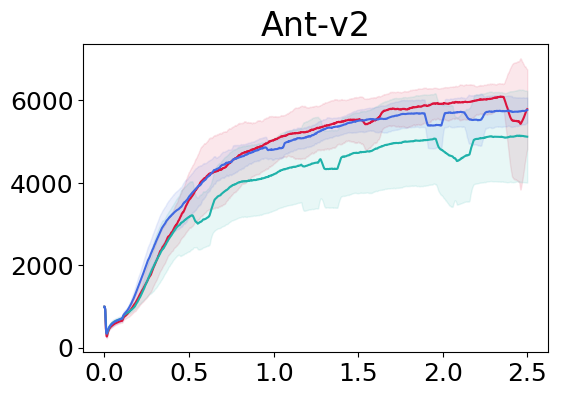}
  \end{subfigure}
  
  \begin{subfigure}{\mainfigsize\paperwidth}
    \includegraphics[width=\linewidth]{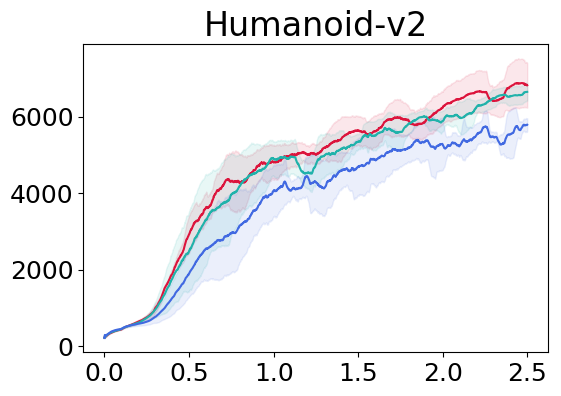}
  \end{subfigure}
    
  \begin{subfigure}{\mainfigsize\paperwidth}
    \includegraphics[width=\linewidth]{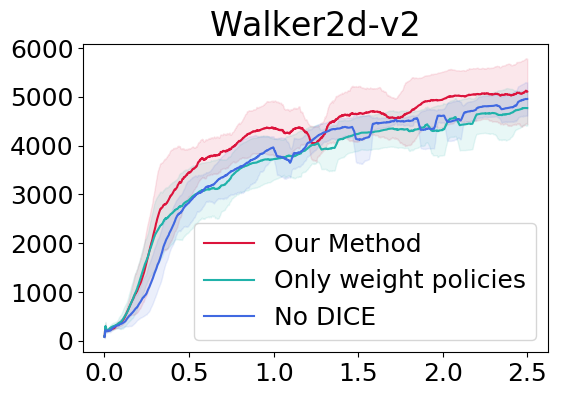}
  \end{subfigure}
  }
\makebox[\textwidth]{
  \begin{subfigure}{\mainfigsize\paperwidth}
    \includegraphics[width=\linewidth]{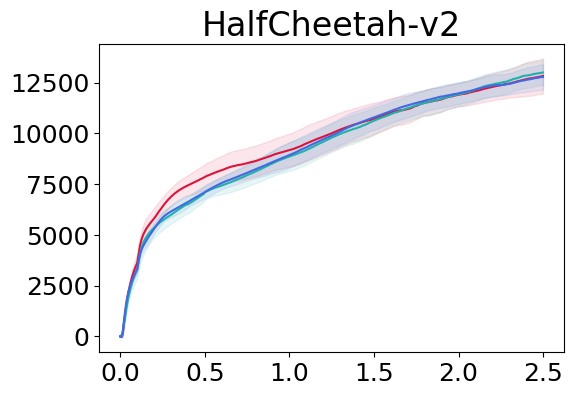}
  \end{subfigure}
  
  \begin{subfigure}{\mainfigsize\paperwidth}
    \includegraphics[width=\linewidth]{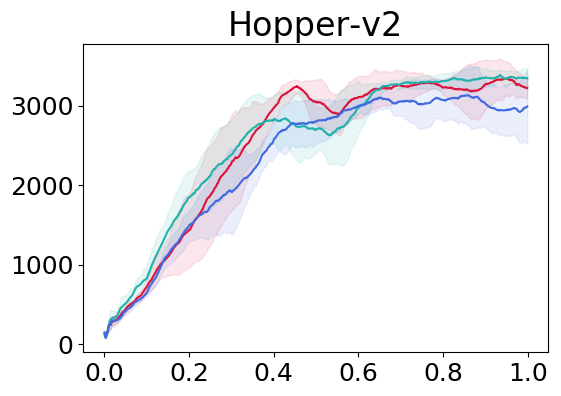}
  \end{subfigure}}
%   \centering{
%       \small{\color{blue}\textbf{---} }: Our Method \qquad  {\color{red}\textbf{---} }: Only weight policies \qquad{\color{green}\textbf{---} }: No DICE
%   } 
  \caption{Ablation on our method on 5 Mujoco environments. 
  x-axis indicates the number of environment steps. 
  y-axis indicates the total undiscounted return.
  The shaded areas denote one standard deviation. Please zoom-in to see more details.}
\label{fig:ablations}
\end{figure*}

We emphasize the key of our framework is to explicitly correct the RL training distribution while training an explore policy solely to correct the flawed Q value that undermines the target policy's gradient estimate. In Section~\ref{sec:avg_return}, we already show that having such an exploration is beneficial. We now show by ablation that such an aggressive training strategy is indeed backed by our distribution correction techniques. The first ablated version of our method, \textbf{No DICE}, does not have the DICE component, i.e., its training procedures can be summarized as Algorithm~\ref{alg} without Line 10-12, and always set $\tilde{\zeta}(s, a) = 1$. As shown in \autoref{fig:ablations}, the performance drops substantially on Humanoid, Walker2d, and Hopper, and somewhat on Ant and HalfCheetah, demonstrating the necessity of performing an effective distribution correction when applying an aggressive exploration strategy.

The second ablated version is to check the necessity of correcting the state-action distribution w.r.t the $Q$ function loss. We term the version of our method without applying the correction ratio to the $Q$ function objective as \textbf{Only weight policies}. From \autoref{fig:ablations}, we can see the performance dropped dramatically on Ant and slightly on the other tasks. We include more ablation studies in the Appendix, showing how different HP influences the performance of our method.

% Consequently, the state distribution of the replay buffer will not differ too much from the stationary state distribution of the target policy. Correcting such small difference will not improve the performance too much, which also explains why state-of-the-art off-policy RL algorithms just ignore such difference when using the target policy to do both exploration and evaluation \citep{SAC,TD3,ddpg}. 

\section{Conclusion}
In this paper, we proposed a novel off-policy RL paradigm that enables the training of an exploration policy that solely focuses on executing the most informative actions without constraining its divergence from the target policy. We support such an efficient exploration by learning the distribution correction ratio to correct the training distribution, which is realized by successfully adapting the DICE \citep{bestdice} optimization schemes into the off-policy actor-critic framework. By evaluating in challenging continuous control tasks, we demonstrate the effectiveness of our method.

\bibliographystyle{unsrtnat}
\bibliography{rl}

\newpage
\appendix

\begin{center}
	{\LARGE {Appendix}}
\end{center}

\section*{Outline of the Appendix } 
In this Appendix, we organize the content in the following ways:
\begin{itemize}
    \item In \autoref{sec:realted-work}, we additional related works.
    \item In \autoref{sec:proof_3p1}, we present the proof of \autoref{prop:correct_dist}.
    % \item In \autoref{sec:justify}, we present further justifications for our second modification to the DICE optimization Problem \eqref{eq:prob_dice}.
    \item In \autoref{sec:hp}, we present experiment details and hyper-parameter settings
     % \item In \autoref{sec:sac}, we further highlight the effectiveness of our learned distribution correction ratio by comparing SAC + DICE vs. SAC.
    \item In \autoref{sec:ablation}, we present additional ablation studies to examine how the hyper-parameters $\beta_{\mathrm{UB}}$, $\beta_{\mathrm{LB}}$ and $T$ impacts the performance. We also show the necessity of correcting the distribution of both the target and exploration policy objectives via comparing with another variant of our approach that only weights the Q function training.
\end{itemize}
\newpage

\section{Additional Related Work}\label{sec:realted-work}
 Incorporating the DICE family ~\citep{nachum2019dualdice, zhang2020gendice, zhang2020gradientdice, uehara2020minimax, nachum2020reinforcement, nachum2019algaedice, bestdice} into the training process of off-policy actor-critic is never easy, as the DICE optimization itself already poses great optimization difficulties. Moreover, in RL, the target policy is always changing and remains unstable. When tackling complicated continuous control tasks with high dimensional state-action space, the instability gets even exacerbated by the limited state-action coverage at the initial training stage~\citep{nachum2019dualdice}. We tackle such difficulties by directly making modifications to the objectives and successfully stabilize the training. Though suffering from bias, experimental results show that our optimization objective actually leads to surprisingly good OPE performance. 
 
In our work, by explicitly correcting the training distribution, we open up the direction to a potentially more profound exploration strategy design. There are a bunch of literature investigate to design better exploration objective ~\citep{brafman2002r, kearns1999efficient, sorg2012variance, lopes2012exploration, araya2012near, doshi2010nonparametric}. While most of the pioneering works mainly tackle the problem in discrete or small state-action space via directly operate on the transition matrix, which cannot be directly extended to complicated continuous control tasks. Noise-based methods~\citep{fortunato2017noisy, plappert2017parameter} improve exploration by perturbing the observation, state, or the parameters in the function approximator for obtaining diverse actions. Intrinsic motivation has also been introduced to encourage the agent to discover novel states~\citep{pathak2017curiosity, burda2018large, stadie2015incentivizing, tang2017exploration, pathak2019self} or better understand the environment~\citep{houthooft2016vime, zhao2019curiosity, burda2018exploration, savinov2018episodic, choshen2018dora, denil2016learning}. Inspired by Thompson sampling, Osband~\etal\citep{osband2016deep} propose to improve exploration of DQN~\citep{mnih2013playing} by modeling uncertainty. Eysenbach~\etal\citep{eysenbach2018diversity} propose to encourage exploration through maximizing the diversity of the learned behavior. SAC~\citep{SAC} tries to improve exploration through adding extra regularization on the entropy term, while OAC~\citep{OAC} derives an upper confidence bound on the state-action value function for facilitating exploration. Although these methods demonstrate that good exploration will benefit the learning process, these works neglect the issue of distribution shift and instability, which is also proven to be a non-trival problem in reinforcement learning~\citep{DISCOR,LFIW}.
\newpage
\section{Proof of  \autoref{prop:correct_dist}}\label{sec:proof_3p1}

\textbf{Proposition \ref{prop:correct_dist}.} $\nabla_{\theta_T}\hat{J}^{\pi_T}$ gives an unbiased estimate of the policy gradient \eqref{eq:policy_grad} with $\pi = \pi_T$.

\begin{proof}
Recall that
\begin{align*}
    \nabla_{\theta} J^{\pi} = \int_{s} d^\pi(s) \bigg(\int_{\varepsilon} \nabla_{\theta} \hat{Q}_{\mathrm{LB}}\left(s, f_{\theta}(s, \varepsilon)\right) \phi(\varepsilon) d \varepsilon + \alpha \int_{\varepsilon}-\nabla_{\theta} \log f_{\theta}(s, \varepsilon) \phi(\varepsilon) d \varepsilon \bigg)d s 
\end{align*}
\begin{align*}
    \hat{J}^{\pi_T} =
    \sum_{(s, a)\in \mathcal{D}} \frac{d^{\pi_T}(s, a)}{d^\mathcal{D}(s, a)}\big( \hat{Q}_{\mathrm{LB}}\left(s, f_{\theta_T}\left(s, \varepsilon\right)\right)
    - \alpha \log f_{\theta_T}\left(s, \varepsilon\right)\big).
\end{align*}

Therefore, $\nabla_{\theta_T}\hat{J}^{\pi_T}$ is an unbiased estimator for
\begin{align*}
     &\quad\int_{s}\int_{a} d^{\mathcal{D}}(s, a) \frac{d^{\pi_T}(s, a)}{d^\mathcal{D}(s, a)} \left(\int_{\varepsilon} \nabla_{\theta_T} \left(\hat{Q}_{\mathrm{LB}}\left(s, f_{\theta_T}(s, \varepsilon)\right) - \alpha\log f_{\theta_T}(s, \varepsilon)\right) \phi(\varepsilon) d \varepsilon \right)d a d s  \\
  &  = \int_{s}\int_{a} d^{\pi_T}(s,a) \left(\int_{\varepsilon} \nabla_{\theta_T} \left(\hat{Q}_{\mathrm{LB}}\left(s, f_{\theta_T}(s, \varepsilon)\right) - \alpha\log f_{\theta_T}(s, \varepsilon)\right) \phi(\varepsilon) d \varepsilon \right) d a d s \\
  &  = \int_{s} d^{\pi_T}(s) \left(\int_{\varepsilon} \nabla_{\theta_T} \left(\hat{Q}_{\mathrm{LB}}\left(s, f_{\theta_T}(s, \varepsilon)\right) - \alpha\log f_{\theta_T}(s, \varepsilon)\right) \phi(\varepsilon) d \varepsilon \right) d s= \nabla_{\theta_T} J^{\pi_T}
\end{align*}
\end{proof}

\clearpage

\section{Hyper-parameters (HP) and Experiment Details}\label{sec:hp}

The implementation of our method is based on OAC's officially released codes\footnote{\label{oac}\url{https://github.com/microsoft/oac-explore}}, which themselves are based on the SAC implementation in \textbf{RLkit} \footnote{\url{https://github.com/rail-berkeley/rlkit}}. We are aware of the fact that OAC's officially released codes$^{\ref{oac}}$ use a set of HP that is slightly different from its original paper. Specifically, they set $\beta_{\mathrm{LB}} = 1.0$ instead of $3.65$ as per (Table 1, Appendix E) of OAC's original paper~\citep{OAC}.

The implementation to solve the DICE optimization is based on the official \textbf{dice\_{rl}} repo\footnote{\url{https://github.com/google-research/dice_rl}}. We set the regularization function  $g_1(x) = g_2(x) = \frac{1}{m} |x|^m$ in Equation \eqref{loss:nu} and \eqref{loss:zeta}, where we treat the exponential $p$ as a hyper-parameter. All of our results are obtained by averaging across $5$ fixed random seeds as per Appendix D of OAC's original paper~\citep{OAC}.

We list the HP specific to our method and those different from SAC and OAC in \autoref{table:hp-our-method}. For completeness, we include the HP for SAC and OAC$^{\ref{oac}}$ in \autoref{table:hp-sac}. We set both $\alpha_\nu = 1$ and $\alpha_\zeta = 1$ as we found this is the only setting to stabilize the DICE optimization in our HP space. We set the temperature $T = 3.0$ as our preliminary experiment finds the value of $T$ from our HP space will not significantly impact the performance, which is proved in our ablation study (Appendix \ref{sec:ablation-T}). As for $\beta_{\mathrm{UB}}$ and $\beta_{\mathrm{LB}}$, we tune their values simultaneously. However, we do not conduct the full sweep in the HP space due to limited computational resources. We empirically find it beneficial to set a similar value for $\beta_{\mathrm{UB}}$ and $\beta_{\mathrm{LB}}$ (Appendix \ref{sec:ablation-beta}).

\begin{table}[H]
\centering
\caption{Hyper-parameters of our method}\label{table:hp-our-method}
\begin{tabular}{ l|l|l }
        \toprule
        { Hyper-parameters} & { Value} & { Range used for tuning} \\ 
        \midrule
        $\beta_{\mathrm{UB}}$  & 2.0  & $\{1.5, 2.0, 2.5, 3.0\}$\\ 

        $\beta_{\mathrm{LB}}$  & 2.5 & $\{1.5, 2.0, 2.5, 3.0\}$ \\ 
        
        Temperature $T$ & 3.0 & \{2.0, 3.0, 5.0\} \\

        Learning rate for $\nu, \zeta, \lambda$ & 0.0001 & Does not tune \\ 
        $\alpha_\nu$   & 1 & \{0, 1\}  \\ 
        $\alpha_\zeta$ & 1 & \{0, 1\}  \\
        discount($\gamma$) for DICE & 0.99  & Does not tune\\
        Regularization function exponential $m$ & 1.5 & Does not tune \\
        $\nu, \zeta$ network hidden sizes & [256, 256] & Does not tune \\
        $\nu,\zeta,\lambda$ optimizer & Adam~\citep{kingma2014adam}  & Does not tune\\
        $\nu,\zeta$ network nonlinearity & ReLU & Does not tune\\
        \bottomrule
        \end{tabular}
\end{table}

\begin{table}[H]
\centering
\caption{Hyper-parameters of SAC and OAC$^{\ref{oac}}$}\label{table:hp-sac}
    \begin{tabular}{ l|l }
        \toprule Parameter & Value \\
        \midrule policy and $Q$ optimizer & Adam~\citep{kingma2014adam} \\
        learning rate for policy and Q & $3 \cdot 10^{-4}$ \\
        discount($\gamma$) & 0.99 \\
        replay buffer size & $10^6$ \\
        policy and $Q$ network hidden sizes & [256, 256]  \\
        minibatch size & 256 \\
        policy and $Q$ nonlinearity & ReLU \\
        target smoothing rate($\tau$) & 0.005 \\
        target update interval & 1 \\
        gradient steps & 1 \\
        \midrule OAC-specific Parameter & Value \\
        \midrule shift multiplier $\sqrt{2\delta}$ & 6.86 \\
        $\beta_{\mathrm{UB}}$  & 4.66 \\
        $\beta_{\mathrm{LB}}$ & 1.0 \\
        \bottomrule
    \end{tabular}
\end{table}

\section{Additional Ablation Study}\label{sec:ablation}
\subsection{Ablation study of the temperature $T$}\label{sec:ablation-T}

\begin{figure*}[h]
  \makebox[\textwidth]{
  
  \begin{subfigure}{\appendixresultfigsize\paperwidth}
    \includegraphics[width=\linewidth]{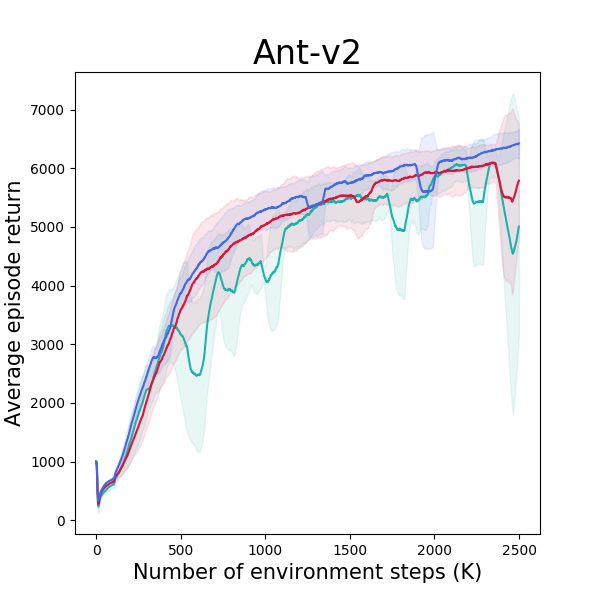}
  \end{subfigure}
  
  \begin{subfigure}{\appendixresultfigsize\paperwidth}
    \includegraphics[width=\linewidth]{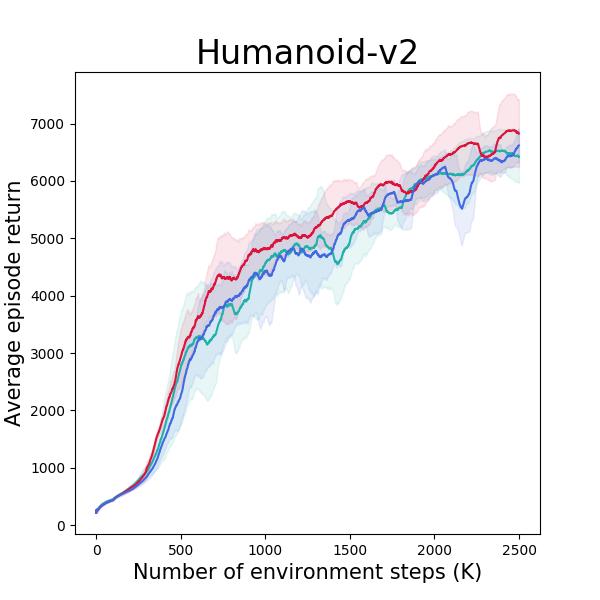}
  \end{subfigure}
    }
  \makebox[\textwidth]{
  \begin{subfigure}{\appendixresultfigsize\paperwidth}
    \includegraphics[width=\linewidth]{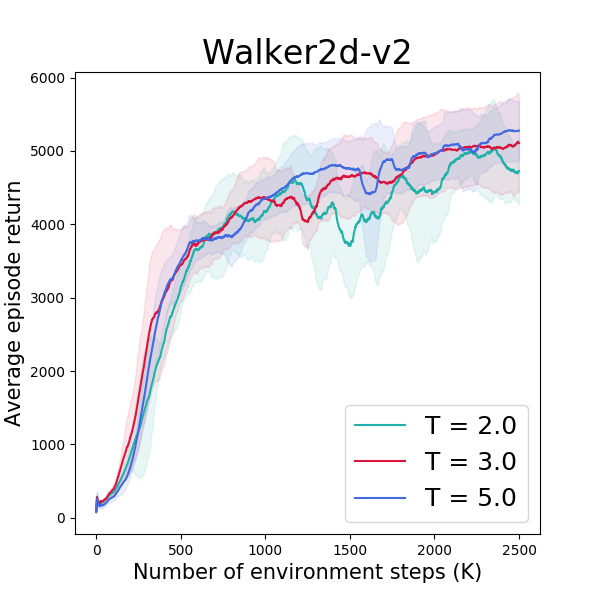}
  \end{subfigure}

  \begin{subfigure}{\appendixresultfigsize\paperwidth}
    \includegraphics[width=\linewidth]{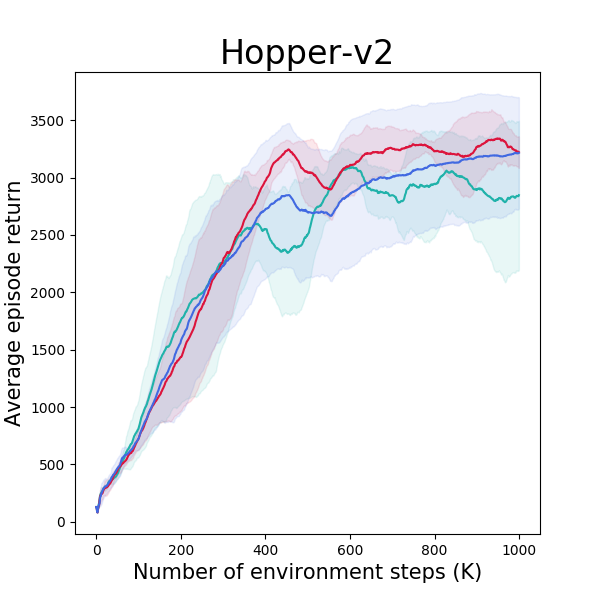}
  \end{subfigure}}

  \caption{Ablation study on the temperature $T$ of our method on 4 Mujoco environments. 
  x-axis indicates the number of environment steps and
  y-axis indicates the total undiscounted return.
  The shaded areas denote one standard deviation between 5 runs of different random seeds.}
\label{fig:ablations-temp}
\end{figure*}

In this section, we examine how the temperature $T$ influences the performance. Note that we keep the other HP unchanged as per \autoref{table:hp-our-method}. For $T\in\{2.0, 3.0, 5.0\}$, we present the performance of Our Method on 4 MuJoCo tasks as shown in $\autoref{fig:ablations-temp}$. A low value of $T$ will introduce instability to the training process, especially for Ant and Walker2d, while the performance on Humanoid remains stable w.r.t different $T$. With $T \geq 3.0$, we can achieve a stable training process and decent final performance in all 4 tasks.
\clearpage

\subsection{Ablation study of $\beta_{\mathrm{UB}}$ and $\beta_{\mathrm{LB}}$}\label{sec:ablation-beta}

\begin{figure*}[h]
  \makebox[\textwidth]{
  
  \begin{subfigure}{\appendixresultfigsize\paperwidth}
    \includegraphics[width=\linewidth]{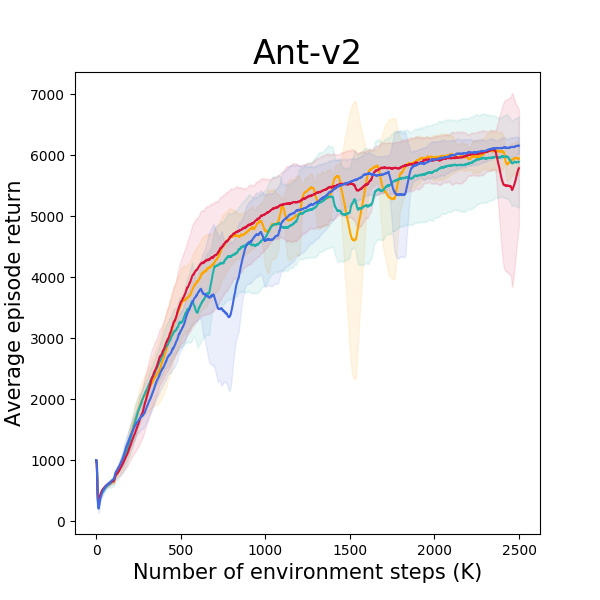}
  \end{subfigure}
  
  \begin{subfigure}{\appendixresultfigsize\paperwidth}
    \includegraphics[width=\linewidth]{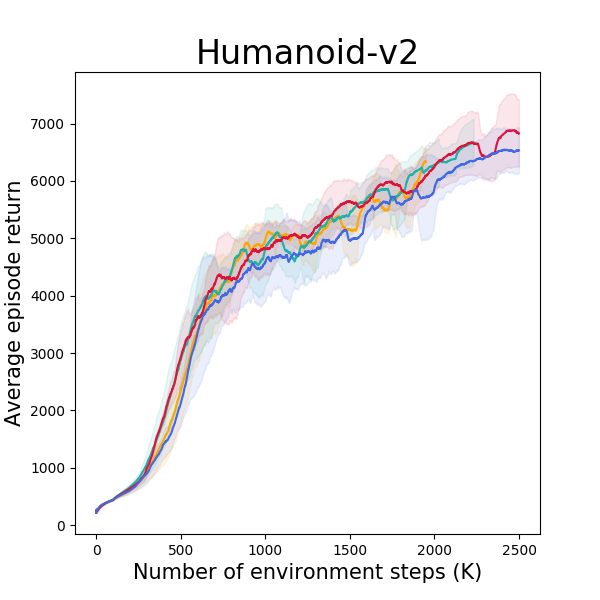}
  \end{subfigure}
 
  \begin{subfigure}{\appendixresultfigsize\paperwidth}
    \includegraphics[width=\linewidth]{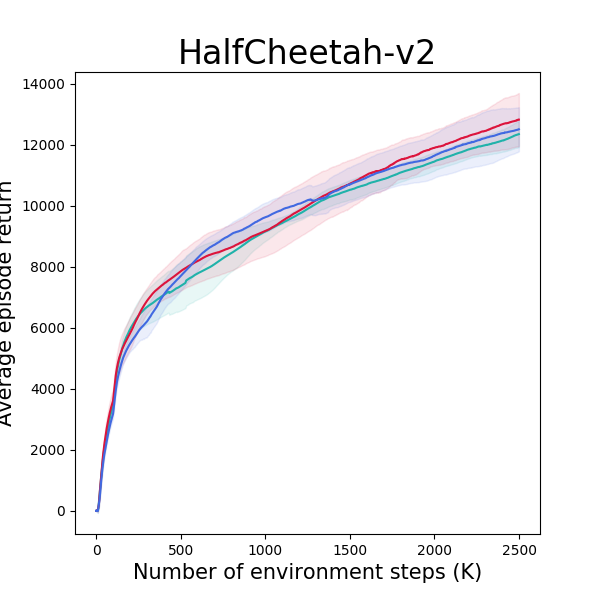}
  \end{subfigure}
    }

  \makebox[\textwidth]{
  \begin{subfigure}{\appendixresultfigsize\paperwidth}
    \includegraphics[width=\linewidth]{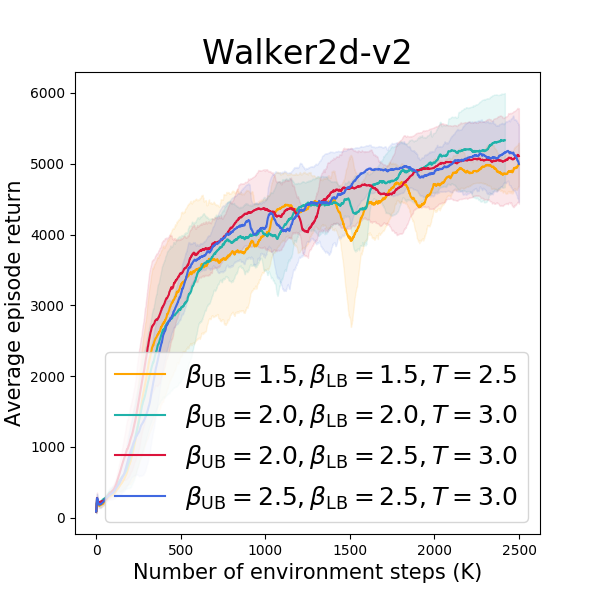}
  \end{subfigure}

  \begin{subfigure}{\appendixresultfigsize\paperwidth}
    \includegraphics[width=\linewidth]{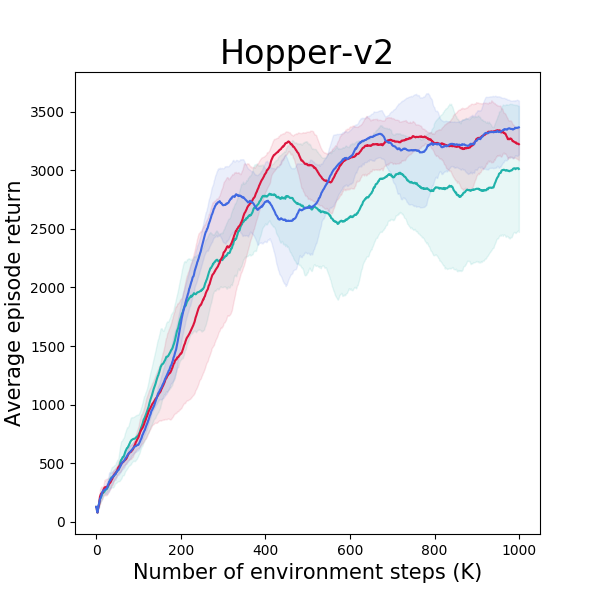}
  \end{subfigure}}

  \caption{Ablation study on the $\beta_{\mathrm{UB}},\beta_{\mathrm{LB}}$ of our method on 4 MuJoco tasks. 
  x-axis indicates the number of environment steps. 
  y-axis indicates the total undiscounted return.
  The shaded areas denote one standard deviation between 5 runs of different random seeds.}
\label{fig:ablations-ub-lb}
\end{figure*}

In this section, we examine how different values of $\beta_{\mathrm{UB}}$ and $\beta_{\mathrm{LB}}$ influence the performance of our method. We empirically find that setting similar values to $\beta_{\mathrm{UB}}$ and $\beta_{\mathrm{LB}}$ generally lead to good performance.

Note that the performance of $\beta_{\mathrm{UB}} = 1.5$, $\beta_{\mathrm{LB}} = 1.5$, and $T = 2.5$ is to show that Our Method is robust to a wide range of $\beta_{\mathrm{UB}}$ and $\beta_{\mathrm{LB}}$ beyonds $\{2.0, 2.5\}$.
\clearpage

\subsection{Ablations on correcting the training distributions of the policies}

\begin{figure*}[h]
  \makebox[\textwidth]{

  \begin{subfigure}{\appendixresultfigsize\paperwidth}
    \includegraphics[width=\linewidth]{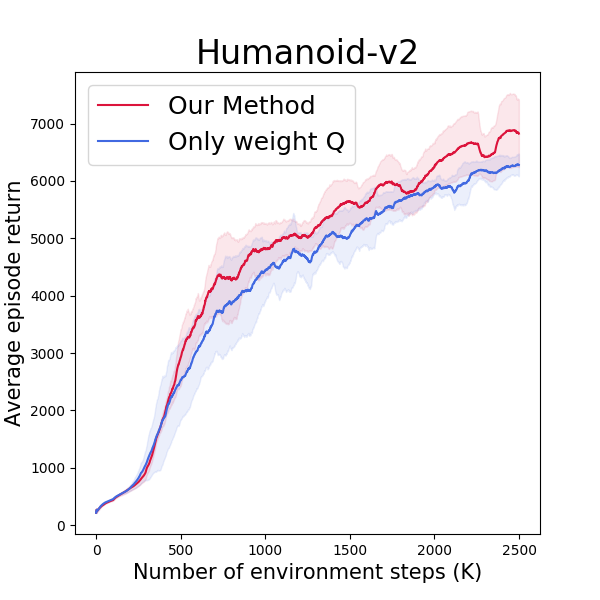}
  \end{subfigure}
 
  \begin{subfigure}{\appendixresultfigsize\paperwidth}
    \includegraphics[width=\linewidth]{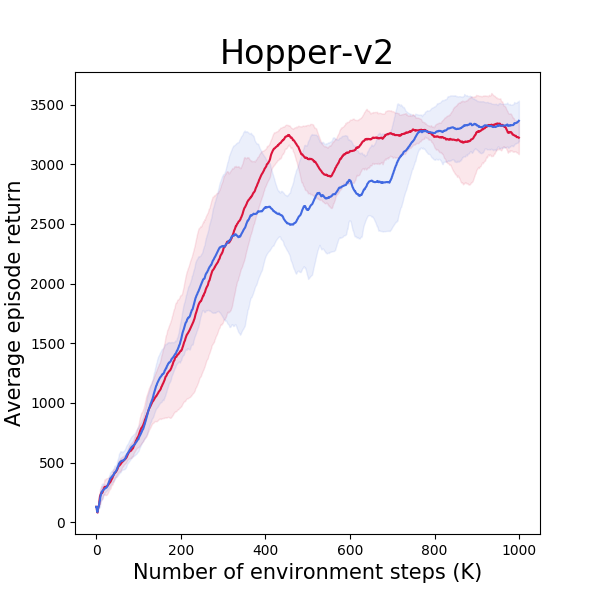}
  \end{subfigure}
    }

  \caption{Ablation study on correcting the training distributions of the policies on 2 MuJoCo tasks. 
  x-axis indicates the number of environment steps. 
  y-axis indicates the total undiscounted return.
  The shaded areas denote one standard deviation between 5 runs of different random seeds.}
\label{fig:ablations-weight-policy}
\end{figure*}

Similar to \textbf{Only weight policies}, we construct the ablated version of our method \textbf{Only weight Q} by removing the distribution correction applied to the target policy's objective \eqref{obj:cor_target} and exploration policy's objective \eqref{obj:cor_explore}. As shown in \autoref{fig:ablations-weight-policy}, ablating the distribution correction for both target and exploration policies of Our Method leads to a clear performance loss on Humanoid. And the ablated version of Our Method also converges slower on the Hopper. Therefore, we can conclude that it is necessary to correct the training distributions of the target and exploration policy.
\clearpage

\end{document}

% If your paper is accepted and the title of your paper is very long,
% the style will print as headings an error message. Use the following
% command to supply a shorter title of your paper so that it can be
% used as headings.
%
%\runningtitle{I use this title instead because the last one was very long}

% If your paper is accepted and the number of authors is large, the
% style will print as headings an error message. Use the following
% command to supply a shorter version of the authors names so that
% they can be used as headings (for example, use only the surnames)
%
%\runningauthor{Surname 1, Surname 2, Surname 3, ...., Surname n}

% Supplementary material: To improve readability, you must use a single-column format for the supplementary material.
\onecolumn
\aistatstitle{Appendix for Off-policy Reinforcement Learning with Optimistic Exploration and Distribution Correction}

\section*{Outline of the Appendix } 
In this Appendix, we organize the content in the following ways:
\begin{itemize}
    \item In \autoref{sec:proof_3p1}, we present the proof of \autoref{main-prop:correct_dist}.
    \item In \autoref{sec:justify}, we present further justifications for our second modification to the DICE optimization Problem \eqref{main-eq:prob_dice}.
    \item In \autoref{sec:hp}, we present experiment details and hyper-parameter settings
     \item In \autoref{sec:sac}, we further highlight the effectiveness of our learned distribution correction ratio by comparing SAC + DICE vs. SAC.
    \item In \autoref{sec:ablation}, we present additional ablation studies to examine how the hyper-parameters $\beta_{\mathrm{UB}}$, $\beta_{\mathrm{LB}}$ and $T$ impacts the performance. We also show the necessity of correcting the distribution of both the target and exploration policy objectives, via comparing with another variant of our approach that only weights the Q function training.
\end{itemize}

\vfill
\clearpage

\section{Proof of  \autoref{main-prop:correct_dist}}\label{sec:proof_3p1}

\textbf{Proposition \ref{main-prop:correct_dist}.} $\nabla_{\theta_T}\hat{J}^{\pi_T}$ gives an unbiased estimate of the policy gradient \eqref{main-eq:policy_grad} with $\pi = \pi_T$.

\begin{proof}
Recall that
\begin{align*}
    \nabla_{\theta} J^{\pi} = \int_{s} d^\pi(s) \bigg(\int_{\varepsilon} \nabla_{\theta} \hat{Q}_{\mathrm{LB}}\left(s, f_{\theta}(s, \varepsilon)\right) \phi(\varepsilon) d \varepsilon + \alpha \int_{\varepsilon}-\nabla_{\theta} \log f_{\theta}(s, \varepsilon) \phi(\varepsilon) d \varepsilon \bigg)d s 
\end{align*}
\begin{align*}
    \hat{J}^{\pi_T} =
    \sum_{(s, a)\in \mathcal{D}} \frac{d^{\pi_T}(s, a)}{d^\mathcal{D}(s, a)}\big( \hat{Q}_{\mathrm{LB}}\left(s, f_{\theta_T}\left(s, \varepsilon\right)\right)
    - \alpha \log f_{\theta_T}\left(s, \varepsilon\right)\big).
\end{align*}

Therefore, $\nabla_{\theta_T}\hat{J}^{\pi_T}$ is an unbiased estimator for
\begin{align*}
     &\quad\int_{s}\int_{a} d^{\mathcal{D}}(s, a) \frac{d^{\pi_T}(s, a)}{d^\mathcal{D}(s, a)} \left(\int_{\varepsilon} \nabla_{\theta_T} \left(\hat{Q}_{\mathrm{LB}}\left(s, f_{\theta_T}(s, \varepsilon)\right) - \alpha\log f_{\theta_T}(s, \varepsilon)\right) \phi(\varepsilon) d \varepsilon \right)d a d s  \\
  &  = \int_{s}\int_{a} d^{\pi_T}(s,a) \left(\int_{\varepsilon} \nabla_{\theta_T} \left(\hat{Q}_{\mathrm{LB}}\left(s, f_{\theta_T}(s, \varepsilon)\right) - \alpha\log f_{\theta_T}(s, \varepsilon)\right) \phi(\varepsilon) d \varepsilon \right) d a d s \\
  &  = \int_{s} d^{\pi_T}(s) \left(\int_{\varepsilon} \nabla_{\theta_T} \left(\hat{Q}_{\mathrm{LB}}\left(s, f_{\theta_T}(s, \varepsilon)\right) - \alpha\log f_{\theta_T}(s, \varepsilon)\right) \phi(\varepsilon) d \varepsilon \right) d s= \nabla_{\theta_T} J^{\pi_T}
\end{align*}
\end{proof}

\clearpage

\section{Justification for the second modification to DICE Optimization Problem (9)}\label{sec:justify}
Recall that the DICE optimization Problem~\eqref{main-eq:prob_dice} is equivalent to the following
\begin{align*}
    \max_{\zeta \geq 0} \min _{\nu, \lambda} L_{D}(\zeta, \nu, \lambda) :=
    & (1-\gamma) \cdot \mathbb{E}_{a_{0} \sim \pi_T\left(s_{0}\right) \atop s_{0}\in \rho_0(s)}\left[\nu\left(s_{0}, a_{0}\right)\right] + \lambda \\ 
    & +\mathbb{E}_{\left(s, a, r, s^{\prime}\right) \sim d^{\mathcal{D}}}\left[\zeta(s, a) \cdot\left(\mathcal{B}^{\pi_T}\nu(s, a)-\nu(s, a)\right) - \lambda\right] \nonumber\\
    & +\alpha_{\nu} \cdot \mathbb{E}_{(s, a) \sim d^{\mathcal{D}}}\left[g_{1}(\nu(s, a))\right] -\alpha_{\zeta} \cdot \mathbb{E}_{(s, a) \sim d^{\mathcal{D}}}\left[g_{2}(\zeta(s, a))\right],
\end{align*}

% First, we provide further explanation for our first modification, i.e., removing $(1-\gamma)\mathbb{E}_{ s_{0}\in \rho_0(s), a_{0} \sim \pi\left(\cdot|s_{0}\right)}\left[\nu\left(s_{0}, a_{0}\right)\right]$ from the objective. Note that setting $\gamma = 1$ will automatically removes this term. However, we still set $\gamma < 1$ for for the term
% $\mathcal{B}^{\pi_T}\nu(s, a) = \alpha_{R} \cdot R(s, a)+\gamma \nu\left(s^{\prime}, a^{\prime}\right), a' \sim \pi_T(s')$ as we empirically find it stabilizing the training.

% To justify our second modification, i.e., 
Our second modification to this objective is
adding an absolute value operator $|\cdot|$ to the single step bellman error $\mathcal{B}^{\pi}\nu(s, a) - \nu(s, a)$. 
While our empirical study finds that it substantially increases the training stability, we now show that this modification does not alter the optimal solution of the original optimization problem without regularization term, i.e., when $\alpha_\nu = \alpha_\zeta = 0$. Formally, our goal in this section is to prove 

\begin{theorem}\label{theo:equivalent}
Solving Problem \eqref{prob:lagrangian}
\begin{equation}\label{prob:lagrangian}
    \max_{\zeta\geq 0} \min _{\nu} L_{D}(\zeta, \nu):=(1-\gamma) \cdot \mathbb{E}_{a_{0} \sim \pi\left(\cdot|s_{0}\right)}\left[\nu\left(s_{0}, a_{0}\right)\right]+\mathbb{E}_{\left(s, a, r, s^{\prime}\right) \sim d^{\prime} \atop a^{\prime} \sim \pi\left(\cdot|s^{\prime}\right)}\left[\zeta(s, a) \cdot\left(r+\gamma \nu\left(s^{\prime}, a^{\prime}\right)-\nu(s, a)\right)\right]
\end{equation}
returns the same solutions as solving Problem \eqref{prob:lagrangian-equ} defined below
\begin{equation}\label{prob:lagrangian-equ}
    \max_{\zeta\geq 0} \min _{Q} L_{D}(\zeta, Q):=(1-\gamma) \cdot \mathbb{E}_{a_{0} \sim \pi\left(\cdot|s_{0}\right)}\left[Q\left(s_{0}, a_{0}\right)\right]+\mathbb{E}_{\left(s, a, r, s^{\prime}\right) \sim d^{\prime} \atop a^{\prime} \sim \pi\left(\cdot|s^{\prime}\right)}\bigg[\zeta(s, a) \cdot\left|r+\gamma Q\left(s^{\prime}, a^{\prime}\right)-Q(s, a)\right|\bigg]
\end{equation}
\end{theorem}
\begin{proof}
The general idea to prove \autoref{theo:equivalent} is to show that Problem \eqref{prob:lagrangian} and Problem \eqref{prob:lagrangian-equ} can be derived from two equivalent sets of optimization problems through Lagrangian.

We start with introducing the derivation of Problem \eqref{prob:lagrangian} from \cite{main-bestdice}, which makes the following assumption:
\begin{assume}[MDP ergodicity \cite{main-bestdice}]\label{assume:ergodicity}
There is unique fixed point solution to
\begin{equation*}
    d^{\pi}(s, a)=(1-\gamma) \rho_{0}(s) \pi(a \mid s)+\gamma \cdot \mathcal{P}_{*}^{\pi} d^{\pi}(s, a), \text { where } \mathcal{P}_{*}^{\pi} d(s, a):=\pi(a \mid s) \sum_{\tilde{s}, \tilde{a}} T(s \mid \tilde{s}, \tilde{a}) d(\tilde{s}, \tilde{a})
\end{equation*}
\end{assume}
To derive Problem \eqref{prob:lagrangian}, the author of \cite{main-bestdice} first reveal the duality between the $Q^\pi$ and $d^\pi$ by \autoref{theo:duality}.

\begin{theorem}[\cite{main-bestdice}]\label{theo:duality}
Given a policy $\pi$, under \autoref{assume:ergodicity}, its normalized expected per-step reward, defined as $\mu(\pi):=(1-\gamma) \mathbb{E}\left[\sum_{t=0}^{\infty} \gamma^{t} R\left(s_{t}, a_{t}\right) \mid s_{0} \sim \rho_{0}, \forall t, a_{t} \sim \pi\left(\cdot|s_{t}\right), s_{t+1} \sim T\left(\cdot|s_{t}, a_{t}\right)\right]$, can be expressed by

\begin{equation}\label{prob:dual-lp}
    \max _{d: S \times A \rightarrow \mathbb{R}} \mathbb{E}_{d}[R(s, a)], \quad \text { s.t., } \quad d(s, a)=(1-\gamma) \rho_{0}(s) \pi(a \mid s)+\gamma \cdot \mathcal{P}_{*}^{\pi} d^{\pi}(s, a)
\end{equation}
Problem \eqref{prob:dual-lp} is termed as dual-LP. Its corresponding primal LP is
\begin{equation}\label{prob:primal-lp}
    \min _{Q: S \times A \rightarrow \mathbb{R}}(1-\gamma) \mathbb{E}_{s_0\sim\rho_{0}, a_0\in\pi(\cdot|s)}[Q(s_0, a_0)], \quad \text { s.t., } \quad Q(s, a)=R(s, a)+\gamma \cdot \mathcal{P}^{\pi} Q(s, a),
\end{equation}
where $\mathcal{P}^{\pi} Q(s, a) = \mathbb{E}_{s^{\prime} \sim T(\cdot|s, a), a^{\prime} \sim \pi\left(\cdot|s^{\prime}\right)}\left[Q\left(s^{\prime}, a^{\prime}\right)\right]$.
\end{theorem}
By approaching Problem \eqref{prob:dual-lp} and \eqref{prob:primal-lp} through the Lagrangian and making change of variables $\zeta(s, a):=d(s, a) / d^{\mathcal{D}}(s, a)$ to enable the use of an arbitrary off-policy distribution $d^{\mathcal{D}}$, we can derive the following optimization problem
\begin{equation}\label{prob:lagrangian-no-pos}
    \max_{\zeta} \min _{\nu} L_{D}(\zeta, \nu):=(1-\gamma) \cdot \mathbb{E}_{a_{0} \sim \pi\left(\cdot|s_{0}\right)}\left[\nu\left(s_{0}, a_{0}\right)\right]+\mathbb{E}_{\left(s, a, r, s^{\prime}\right) \sim d^{\prime} \atop a^{\prime} \sim \pi\left(\cdot|s^{\prime}\right)}\left[\zeta(s, a) \cdot\left(r+\gamma \nu\left(s^{\prime}, a^{\prime}\right)-\nu(s, a)\right)\right]
\end{equation}

Note that solving Problem \eqref{prob:lagrangian-no-pos} above gives us the correction ratio we want $\zeta^*(s, a) = \frac{d^{\pi}(s, a)}{d^\mathcal{D}(s, a)} \geq 0$. Thus, the author of \cite{main-bestdice} further add the constraint $\zeta > 0$ to Problem \eqref{prob:lagrangian-no-pos} without affecting its optimal solution, which results in Problem \eqref{prob:lagrangian}.

Following a similar recipe as above, Problem \eqref{prob:lagrangian-equ} can be derived by approaching Problem \eqref{prob:dual-lp} and the following \eqref{prob:primal-lp-equ}, %defined below, 
through the Lagrangian and making change of variables $\zeta(s, a):=d(s, a) / d^{\mathcal{D}}(s, a)$
\begin{equation}\label{prob:primal-lp-equ}
    \min _{Q: S \times A \rightarrow \mathbb{R}}(1-\gamma) \mathbb{E}_{s_0\sim\rho_{0}, a_0\in\pi(\cdot|s)}[Q(s_0, a_0)], \quad \text { s.t., } \quad \left|R(s, a)+\gamma \cdot \mathcal{P}^{\pi} Q(s, a) - Q(s, a)\right| = 0
\end{equation}
At this point, the only thing left to show is that solving Problem \eqref{prob:primal-lp} returns the same solution as solving Problem \eqref{prob:primal-lp-equ}. Such an equivalence is obvious, as we have
\begin{equation}
    Q(s, a) = R(s, a)+\gamma \cdot \mathcal{P}^{\pi} Q(s, a) \iff 
    \left|R(s, a)+\gamma \cdot \mathcal{P}^{\pi} Q(s, a) - Q(s, a)\right| = 0
\end{equation}

We thus claim that adding an absolute value operator $|\cdot|$ to the single step bellman error $\mathcal{B}^{\pi}\nu(s, a) - \nu(s, a)$ of Problem \eqref{main-eq:prob_dice} will not affects its validality in terms of deriving the correction ratio.
\end{proof}

\clearpage

\section{Hyper-parameters (HP) and Experiment Details}\label{sec:hp}

The implementation of our method is based on OAC's officially released codes\footnote{\label{oac}\url{https://github.com/microsoft/oac-explore}}, which themselves are based on the SAC implementation in \textbf{RLkit} \footnote{\url{https://github.com/rail-berkeley/rlkit}}, a widely used RL library with reputation. We are aware of the fact that OAC's officially released codes$^{\ref{oac}}$ use a set of HP that is slightly different from its original paper. Specifically, they set $\beta_{\mathrm{LB}} = 1.0$ instead of $3.65$ as per (Table 1, Appendix E) of OAC's original paper~\cite{OAC}.

The implementation to solve the DICE optimization is based on the official \textbf{dice\_{rl}} repo\footnote{\url{https://github.com/google-research/dice_rl}}. We set the regularization function  $g_1(x) = g_2(x) = \frac{1}{m} |x|^m$ in Equation \eqref{main-loss:nu} and \eqref{main-loss:zeta}, where we treat the exponential $p$ as a hyper-parameter. All of our results are obtained by averaging across $5$ fixed random seeds as per Appendix D of OAC's original paper~\cite{OAC}.

We list the HP that are specific to our method and those different from both SAC and OAC in \autoref{table:hp-our-method}. For completeness, we include the HP for SAC and OAC$^{\ref{oac}}$ in \autoref{table:hp-sac}. We set both $\alpha_\nu = 1$ and $\alpha_\zeta = 1$ as we found this is the only setting to stabilize the DICE optimization in our HP space. We set the temperature $T = 3.0$ as our preliminary experiment finds the value of $T$ from our HP space will not significantly impact the performance, which is proved in our ablation study (\autoref{sec:ablation-T}). As for $\beta_{\mathrm{UB}}$ and $\beta_{\mathrm{LB}}$, we tune their values simultaneously. However, we do not conduct the full sweep in the HP space due to limited computational resources. We empirically find it beneficial to set a similar value for $\beta_{\mathrm{UB}}$ and $\beta_{\mathrm{LB}}$ (\autoref{sec:ablation-beta}).

\begin{table}[H]
\centering
\caption{Hyper-parameters of our method}\label{table:hp-our-method}
\begin{tabular}{ |l|l|l| }
        \hline
        {\bf Hyper-parameters} & {\bf Value} & {\bf Range used for tuning} \\ 
        \hline
        $\beta_{\mathrm{UB}}$  & 2.0  & $\{1.5, 2.0, 2.5, 3.0\}$\\ 

        $\beta_{\mathrm{LB}}$  & 2.5 & $\{1.5, 2.0, 2.5, 3.0\}$ \\ 
        
        Temperature $T$ & 3.0 & \{2.0, 3.0, 5.0\} \\

        Learning rate for $\nu, \zeta, \lambda$ & 0.0001 & Does not tune \\ 
        $\alpha_\nu$   & 1 & \{0, 1\}  \\ 
        $\alpha_\zeta$ & 1 & \{0, 1\}  \\
        discount($\gamma$) for DICE & 0.99  & Does not tune\\
        Regularization function exponential $m$ & 1.5 & Does not tune \\
        $\nu, \zeta$ network hidden sizes & [256, 256] & Does not tune \\
        $\nu,\zeta,\lambda$ optimizer & Adam~\cite{kingma2014adam}  & Does not tune\\
        $\nu,\zeta$ network nonlinearity & ReLU & Does not tune\\
        \hline
        \end{tabular}
\end{table}

\begin{table}[H]
\centering
\caption{Hyper-parameters of SAC and OAC$^{\ref{oac}}$}\label{table:hp-sac}
    \begin{tabular}{ |l|l| }
        \hline Parameter & Value \\
        \hline policy and $Q$ optimizer & Adam~\cite{kingma2014adam} \\
        learning rate for policy and Q & $3 \cdot 10^{-4}$ \\
        discount($\gamma$) & 0.99 \\
        replay buffer size & $10^6$ \\
        policy and $Q$ network hidden sizes & [256, 256]  \\
        minibatch size & 256 \\
        policy and $Q$ nonlinearity & ReLU \\
        target smoothing rate($\tau$) & 0.005 \\
        target update interval & 1 \\
        gradient steps & 1 \\
        \hline OAC-specific Parameter & Value \\
        \hline shift multiplier $\sqrt{2\delta}$ & 6.86 \\
        $\beta_{\mathrm{UB}}$  & 4.66 \\
        $\beta_{\mathrm{LB}}$ & 1.0 \\
        \hline
    \end{tabular}
\end{table}

\section{SAC + DICE vs. SAC}\label{sec:sac}

\begin{figure*}[h]
  \makebox[\textwidth]{
  
  \begin{subfigure}{\resultfigsize\paperwidth}
    \includegraphics[width=\linewidth]{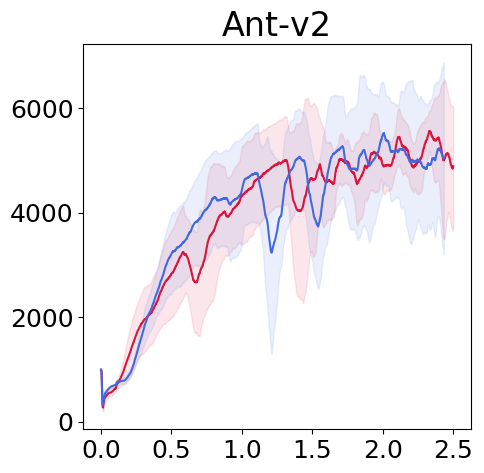}
  \end{subfigure}
  
  \begin{subfigure}{\resultfigsize\paperwidth}
    \includegraphics[width=\linewidth]{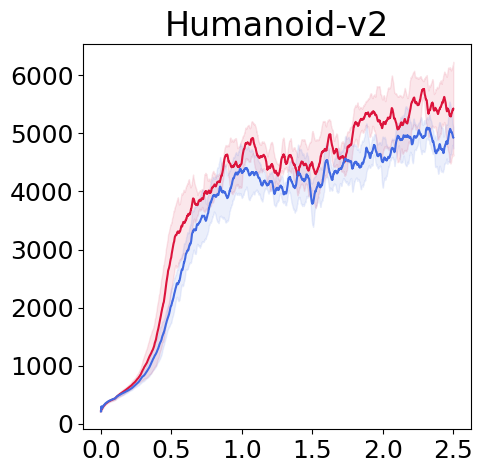}
  \end{subfigure}
    }
  \makebox[\textwidth]{
  \begin{subfigure}{\resultfigsize\paperwidth}
    \includegraphics[width=\linewidth]{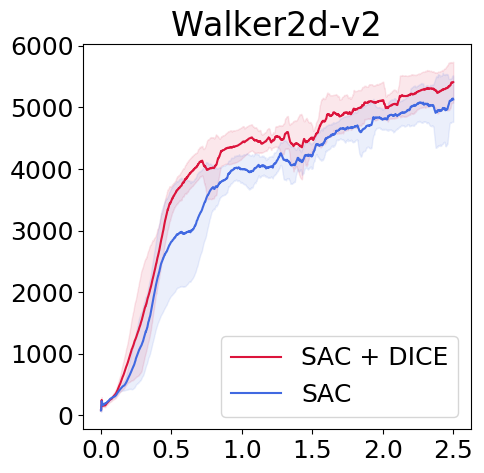}
  \end{subfigure}

  \begin{subfigure}{\resultfigsize\paperwidth}
    \includegraphics[width=\linewidth]{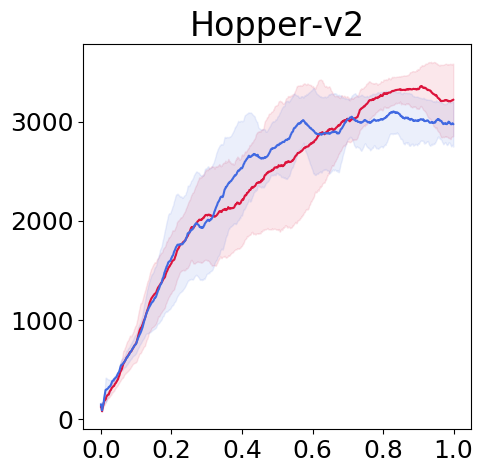}
  \end{subfigure}}

  \caption{Comparing SAC + DICE vs. SAC on 4 Mujoco environments. 
  x-axis indicates the number of environment steps. 
  y-axis indicates the total undiscounted return.
  The shaded areas denote one standard deviation between 5 runs of different random seeds.}
\label{fig:sac-dice}
\end{figure*}
Like OAC + DICE, we implement SAC + DICE by integrating our DICE correction ratio into SAC without any HP tuning. As shown in \autoref{fig:sac-dice}, SAC + DICE outperforms SAC on the challenging Humanoid and Walker2d by a clear margin while being on par with SAC on the Ant. There is also a slight performance gain on the Hopper.

% Until this point, 
It is worth noted that integrating our learned DICE correction ratio into \textbf{NO DICE} (variant of ours), OAC, and SAC all results in a performance gain. We thus highlight that our proposed DICE correction scheme
effectively benefits general off-policy actor-critic training.

\clearpage
\section{Additional Ablation Study}\label{sec:ablation}
\subsection{Ablation study of the temperature $T$}\label{sec:ablation-T}

\begin{figure*}[h]
  \makebox[\textwidth]{
  
  \begin{subfigure}{\resultfigsize\paperwidth}
    \includegraphics[width=\linewidth]{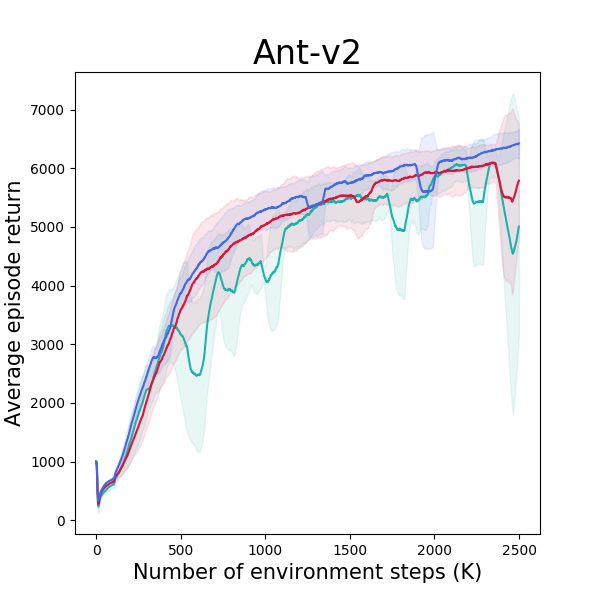}
  \end{subfigure}
  
  \begin{subfigure}{\resultfigsize\paperwidth}
    \includegraphics[width=\linewidth]{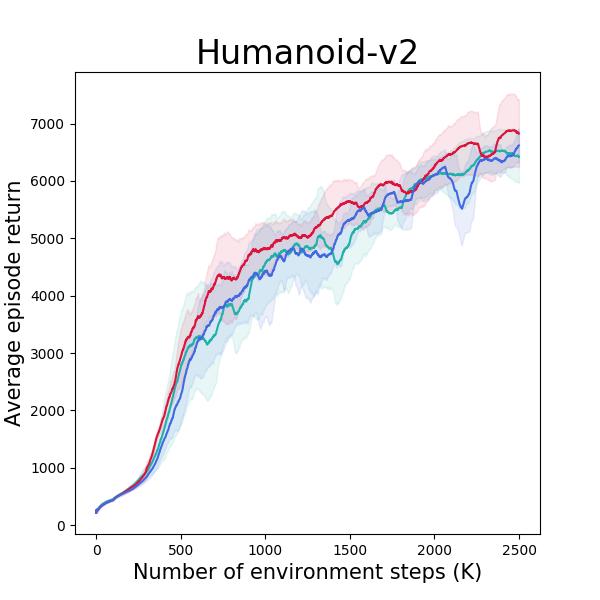}
  \end{subfigure}
    }
  \makebox[\textwidth]{
  \begin{subfigure}{\resultfigsize\paperwidth}
    \includegraphics[width=\linewidth]{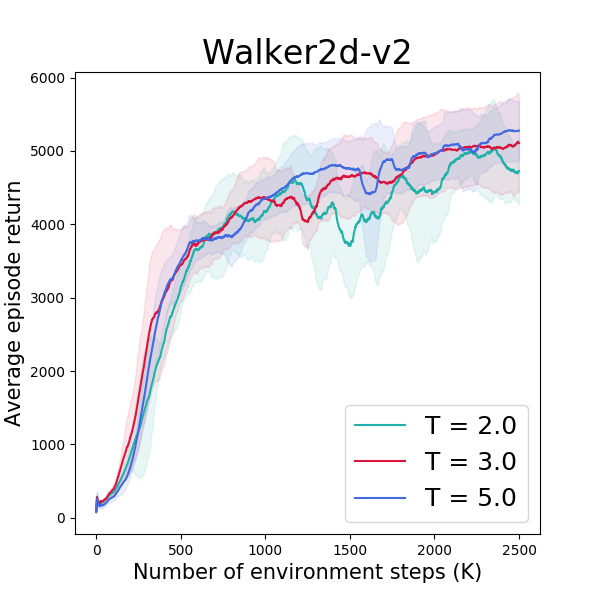}
  \end{subfigure}

  \begin{subfigure}{\resultfigsize\paperwidth}
    \includegraphics[width=\linewidth]{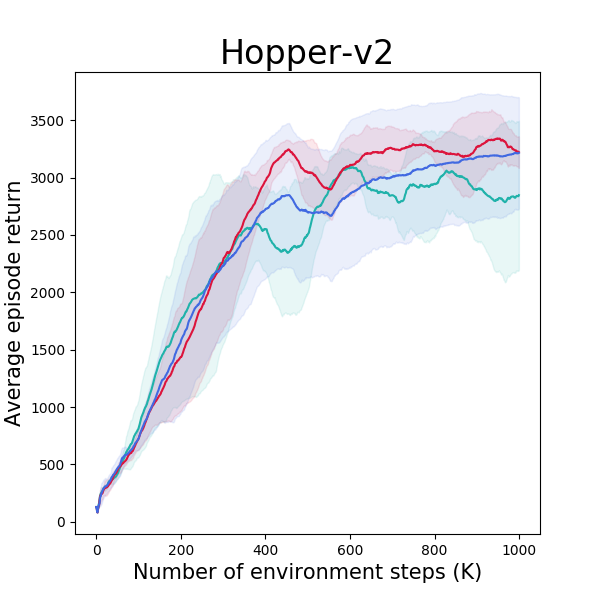}
  \end{subfigure}}

  \caption{Ablation study on the temperature $T$ of our method on 4 Mujoco environments. 
  x-axis indicates the number of environment steps and
  y-axis indicates the total undiscounted return.
  The shaded areas denote one standard deviation between 5 runs of different random seeds.}
\label{fig:ablations-temp}
\end{figure*}

In this section, we examine how the temperature $T$ influences the performance. Note that we keep the other HP unchanged as per \autoref{table:hp-our-method}. For $T\in\{2.0, 3.0, 5.0\}$, we present the performance of Our Method on 4 MuJoCo tasks as shown in $\autoref{fig:ablations-temp}$. A low value of $T$ will introduce instability to the training process, especially for Ant and Walker2d, while the performance on Humanoid remains stable w.r.t different $T$. With $T \geq 3.0$, we can achieve a stable training process and decent final performance in all 4 tasks.
\clearpage

\subsection{Ablation study of $\beta_{\mathrm{UB}}$ and $\beta_{\mathrm{LB}}$}\label{sec:ablation-beta}

\begin{figure*}[h]
  \makebox[\textwidth]{
  
  \begin{subfigure}{\resultfigsize\paperwidth}
    \includegraphics[width=\linewidth]{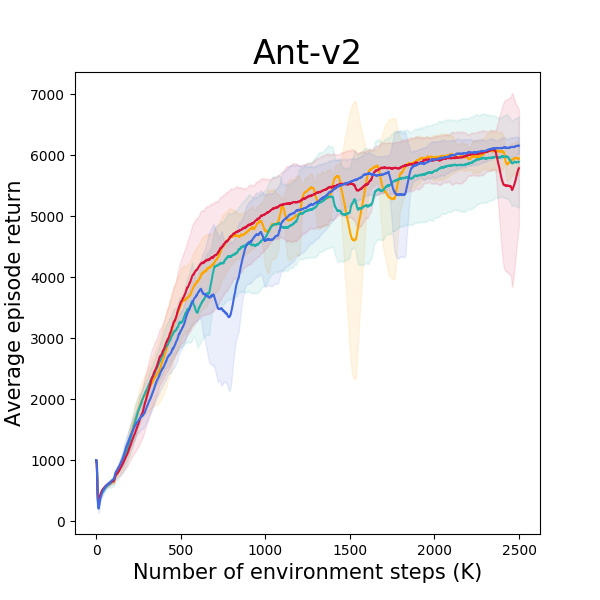}
  \end{subfigure}
  
  \begin{subfigure}{\resultfigsize\paperwidth}
    \includegraphics[width=\linewidth]{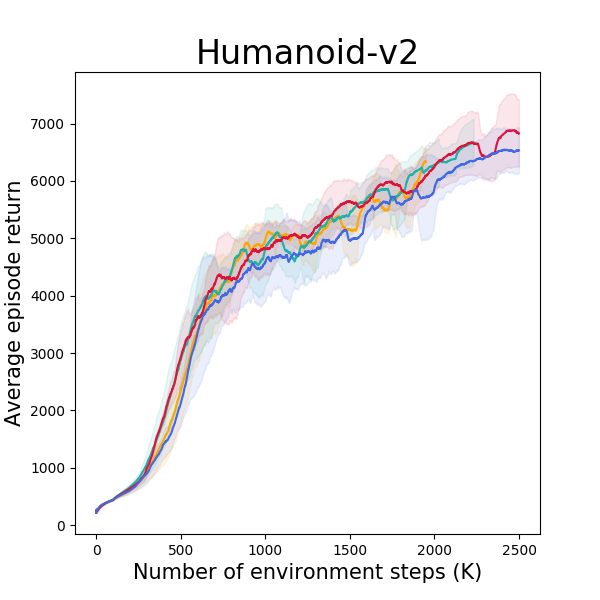}
  \end{subfigure}
 
  \begin{subfigure}{\resultfigsize\paperwidth}
    \includegraphics[width=\linewidth]{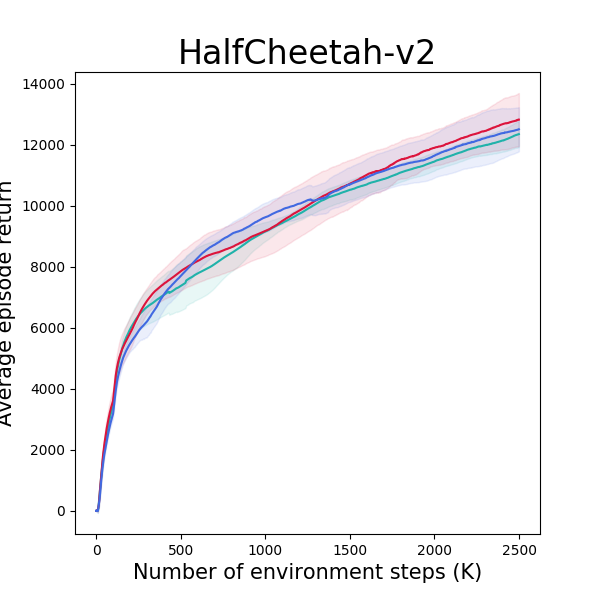}
  \end{subfigure}
    }

  \makebox[\textwidth]{
  \begin{subfigure}{\resultfigsize\paperwidth}
    \includegraphics[width=\linewidth]{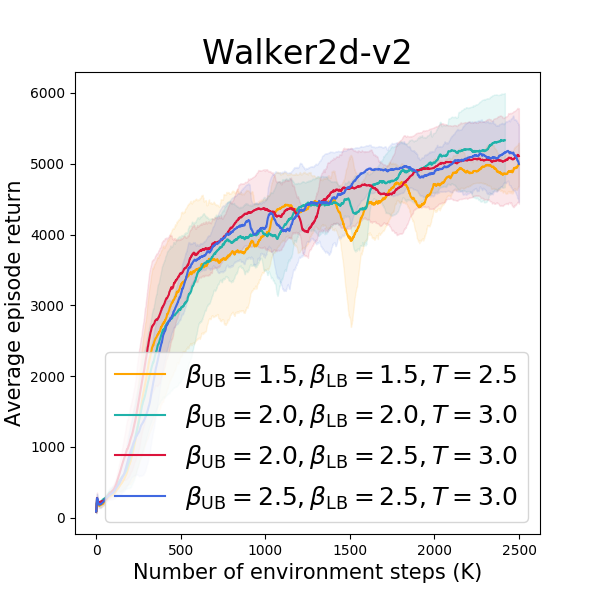}
  \end{subfigure}

  \begin{subfigure}{\resultfigsize\paperwidth}
    \includegraphics[width=\linewidth]{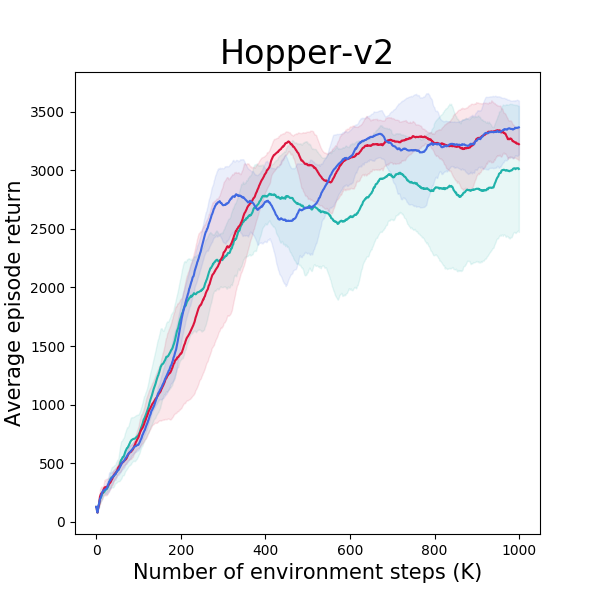}
  \end{subfigure}}

  \caption{Ablation study on the $\beta_{\mathrm{UB}},\beta_{\mathrm{LB}}$ of our method on 4 MuJoco tasks. 
  x-axis indicates the number of environment steps. 
  y-axis indicates the total undiscounted return.
  The shaded areas denote one standard deviation between 5 runs of different random seeds.}
\label{fig:ablations-ub-lb}
\end{figure*}

In this section, we examine how different values of $\beta_{\mathrm{UB}}$ and $\beta_{\mathrm{LB}}$ influence the performance of our method. We empirically find that setting similar values to $\beta_{\mathrm{UB}}$ and $\beta_{\mathrm{LB}}$ generally lead to good performance.

Note that the performance of $\beta_{\mathrm{UB}} = 1.5$, $\beta_{\mathrm{LB}} = 1.5$, and $T = 2.5$ is to show that Our Method is robust to a wide range of $\beta_{\mathrm{UB}}$ and $\beta_{\mathrm{LB}}$ beyonds $\{2.0, 2.5\}$.
\clearpage

\subsection{Ablations on correcting the training distributions of the policies}

\begin{figure*}[h]
  \makebox[\textwidth]{

  \begin{subfigure}{\resultfigsize\paperwidth}
    \includegraphics[width=\linewidth]{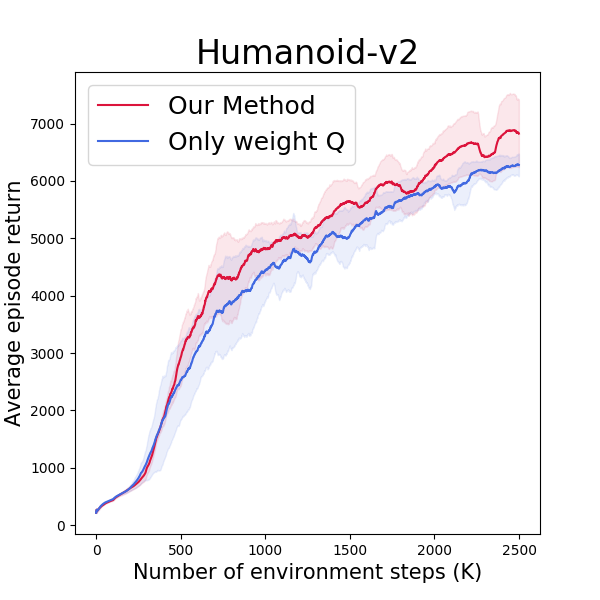}
  \end{subfigure}
 
  \begin{subfigure}{\resultfigsize\paperwidth}
    \includegraphics[width=\linewidth]{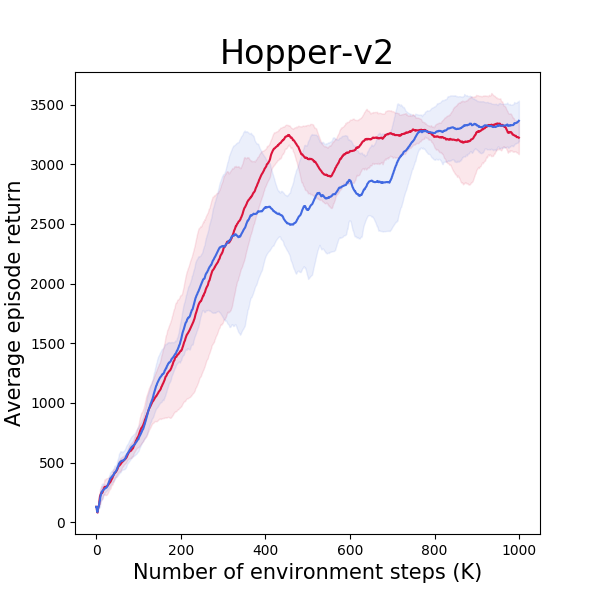}
  \end{subfigure}
    }

  \caption{Ablation study on correcting the training distributions of the policies on 2 MuJoCo tasks. 
  x-axis indicates the number of environment steps. 
  y-axis indicates the total undiscounted return.
  The shaded areas denote one standard deviation between 5 runs of different random seeds.}
\label{fig:ablations-weight-policy}
\end{figure*}

Similar to \textbf{Only weight policies}, we construct the ablated version of our method \textbf{Only weight Q} by removing the distribution correction applied to the target policy's objective \eqref{main-obj:cor_target} and exploration policy's objective \eqref{main-obj:cor_explore}. As shown in \autoref{fig:ablations-weight-policy}, ablating the distribution correction for both target and exploration policies of Our Method leads to a clear performance loss on Humanoid. And the ablated version of Our Method also converges slower on the Hopper. Therefore, we can conclude that it is necessary to correct the training distributions of the target and exploration policy.
\clearpage

\bibliographystyle{unsrt}
\bibliography{rl}